\documentclass[twoside,11pt]{article}
\usepackage{mathrsfs}
\usepackage{amsfonts}
\usepackage{bbm}
\usepackage{algorithm}
\usepackage{graphics}
\usepackage{epsfig}
\usepackage{color}
\usepackage{algorithmic}
\usepackage{subfigure}
\usepackage{color}

\usepackage{jmlr2e}
\usepackage{mathrsfs,multirow,hhline}

\newtheorem{assumption}{Assumption}

\begin{document}

\title{ Kernel-based $L_2$-Boosting  with Structure Constraints}

\author{\name Yao Wang \email yao.s.wang@gmail.com\\
         \addr  Center of Intelligent Decision-Making and Machine Learning\\
         School of Management
         \\ Xi'an
        Jiaotong University,
           Xi'an, China\\
\name Xin Guo \email x.guo@polyu.edu.hk \\
        \addr Department of Applied Mathematics \\
         The Hong Kong Polytechnic University,
           Kowloon, Hong Kong \\
\name Shao-Bo Lin\thanks{Corresponding author}  \email sblin1983@gmail.com \\
        \addr Center of Intelligent Decision-Making and Machine Learning\\
        School of Management\\ Xi'an
        Jiaotong University,
           Xi'an, China }

\editor{???}


\maketitle



\begin{abstract}%
Developing efficient kernel methods for regression is very popular in the past decade.
In this paper, utilizing boosting on kernel-based weaker learners, we propose a novel kernel-based learning algorithm called
kernel-based re-scaled boosting with truncation, dubbed as KReBooT. The proposed KReBooT benefits in controlling the structure  of estimators and producing sparse estimate, and is near overfitting resistant. We conduct both theoretical analysis and numerical simulations to illustrate the power of KReBooT. Theoretically,
we prove that  KReBooT can achieve the almost optimal numerical convergence rate for nonlinear approximation. Furthermore, using the recently developed integral operator approach and  a variant of Talagrand's concentration inequality, we provide  fast learning rates for KReBooT, which is a new record of boosting-type algorithms. Numerically, we carry out a series of simulations to show the promising performance of KReBooT in terms of its good generalization, near over-fitting resistance and structure constraints.
\end{abstract}

\begin{keywords}
Learning theory, kernel methods, boosting,
re-scaling, truncation
\end{keywords}

\section{Introduction}

In a regression problem, data of input-output pairs are given to feed the learning  with the purpose of modeling   the relationship between inputs and outputs.
Kernel methods \citep{Evgeniou2000}, which map input points from the input space to some kernel-based feature space to make the learning method be linear, have been widely used  for regression in the last two decades.
Learning algorithms including kernel ridge regression \citep{Caponnetto2007}, kernel-based gradient descent \citep{Yao2007},  kernel-based spectral algorithms \citep{Gerfo2008}, kernel-based conjugate gradient algorithms \citep{Blanchard2016} and kernel-based LASSO \citep{Wang2007}   have been  proposed for regression with perfect  feasibility verifications.  To attack the design flaw of kernel methods in computation, several variants
such as distributed learning \citep{Zhang2015}, localized learning \citep{Meister2016} and learning with sub-sampling \citep{Grittens2016}, have been developed to derive scalable kernel-based learning algorithms and been successfully used in numerous massive data regression problems.

Our purpose is not to pursue novel scalable variants of  kernel methods to tackle massive data, but to present  novel  kernel-based learning algorithms  to realize different utility of data. This is a hot topic in recent years and
numerous novel kernel-based learning algorithms have been proposed for different purpose. In particular, \citep{Lin2018b} proposed the  kernel-based partial least squares to accelerate the convergence rate of kernel-based gradient descent and provided optimal learning rate verifications; \citep{Guo2017a} developed a kernel-based threshold algorithms to derive sparse estimator to enhance the interpretability and reduce testing time; \citep{Guo2017c} proposed a bias corrected regularization kernel network to reduce the bias of kernel ridge regression; and more recently \citep{Lin2019JMLR} combined the well known $L_2$-Boosting \citep{Buhlmann2003} with kernel ridge regression to avoid the saturation of kernel ridge regression and reduce the difficulty of parameter-selection.

Besides the generalization capability  \citep{Evgeniou2000}, three important factors affecting the learning performance of a kernel-based   algorithm  are the computational complexity, parameter selection and interpretability. The computational complexity \citep{Rudi2015} reflects the time price of a learning algorithm in a single trail; the parameter selection \citep{Caponnetto2010} frequently refers to the  number of trails in the learning process and the interpretability in the framework of kernel learning \citep{Shi2011} usually concerns the sparseness of the derived estimator.  Our basic idea is to combine kernel methods with a new variant of boosting to derive a learning algorithms that is user-friendly, over-fitting resistant and well interpretable. Taking a set of kernel functions as the weak learners, the variant of boosting combines the ideas of {\it regularization} in \citep{Zhang2005} and {\it re-scaling} in \citep{Wang2019}.

The idea of {\it Regularization} aims at controlling the step-size of  boosting
iterations and   derives estimators with structure constraints (relatively small $\ell_1$
norm).  Regularized boosting
via truncation (RTboosting) \citep{Zhang2005} is a typical variant of boosting based on {\it regularization}. RTboosting controls the step-size in each boosting iteration via limiting the range of linear search, which succeeds in improving the performance of boosting and enhancing the interpretability.
 However,   to the best of our knowledge,  fast
numerical convergence rates  were not provided for these variants. In particular, it can be found in \citep{Zhang2005} that the numerical convergence of RTboosting is of an order $\mathcal O(k^{-1/3})$, which is far worse than the optimal rate for  nonlinear approximation $\mathcal O(k^{-1})$ \citep{DeVore1996}. Here and hereafter, $k$ denotes the number of boosting iterations.
 The idea
of  {\it re-scaling}  focuses on multiplying a  re-scaling  parameter to the
estimator of each boosting iteration to accelerate the numerical
convergence rate of boosting.
In particular,
re-scaled boosting (Rboosting) \citep{Wang2019}  which shrinks the
estimator obtained in the previous boosting iteration was shown to   achieve the optimal numerical convergence rate under certain sparseness assumption.  The problem is, however,  there aren't any
guarantees for the structure ($\ell_1$ norm) of the estimator, making the strategy lack of interpretability.

In this paper,
we propose a novel kernel-based re-scaled boosting with truncation (KReBooT)  to embody advantages of  {\it regularization}  and  {\it re-scaling}  simultaneously. We find a close relation between the regularization parameter and re-scaling parameter to accelerate the numerical convergence and control the $\ell_1$ norm of the derived estimator.
This together with the well developed integral operator technique in \citep{Lin2017,Guo2017} and a Talagrand's concentration inequality \citep{Steinwart2008} yields an almost optimal numerical convergence rate and a fast learning rate of KReBooT. In particular, the new algorithm  can achieve a learning rate as
far as $\mathcal O( \frac1m\log^2m)$ under some standard assumptions to the
kernel,  where $m$ is the number of training samples. Due to the structure constraint again,
we also prove that the new algorithm is almost
overfitting resistant in the sense that the bias decreases inversely
proportional to $k$, while the variance increases logarithmical with
respect to $k$.
Finally, it should be mentioned that there
  are totally  three types of
parameters including   re-scaling  parameters,  regularization parameters  and  iteration numbers involved in the new algorithm.  Our theoretical analysis  shows that the learning performance is not sensitive to them, making the algorithm to be user-friendly. In fact, the re-scaling and regularization parameters can be determined before the learning process and the number of iterations can be selected to be relatively large, since KReBooT is almost overfitting resistant.
  We conduct a series of   numerical simulations  to
illustrate the outperformance of the new algorithm, compared with widely used kernel methods. The numerical
results are consistent with our theoretical claims and therefore
verify  our assertions.

The rest of paper is organized as follows. In the next section, we
introduce detailed implementation of KReBooT.
Section \ref{Sec.Properties} provides convergence guarantees for KReBooT as well as its almost optimal numerical convergence rate.
 In Section \ref{Sec.Learning}, we derive fast learning rate for KReBooT in the framework of learning theory.
 Section
\ref{Sec.Experiments} presents the numerical verifications for our
theoretical assertions. In Section \ref{Sec.Proof}, we prove our
main results.

\section{Kernel-based  Re-scaled Boosting with Truncation}\label{Sec.Kboosting}
%

Let $D=\{z_i\}_{i=1}^m=\{(x_i,y_i)\}_{i=1}^m$ be the
set of samples with $x_i\in\mathcal X$ and $y_i\in\mathcal Y$, where $\mathcal X$ is a compact input space and $\mathcal Y\subseteq [-M,M]$ is the output space for some $M>0$. Given
a Mercer kernel $K: \mathcal X \times \mathcal X \to \mathbb R$,
denote by $(\mathcal H_K,\|\cdot\|_K)$ the corresponding reproducing
kernel Hilbert space (RKHS). The compactness of $\mathcal X$ implies $\kappa:=\sqrt{\sup_{x\in \mathcal X}K(x,x)}\leq \infty$. Throughout this paper, we assume $\kappa\leq 1$ for the sake of brevity.
Set $S:=\{K_{x_i}:i=1,\dots,m\}$   with
$K_x(\cdot)=K(\cdot,x)$.
 Let
$
             \mathcal
             H_{K,D}:=\left\{\sum_{i=1}^ma_iK_{x_i}:a_i\in\mathbb
             R\right\}.
$
The well known representation theorem \citep{Cucker2007} shows that all the aforementioned kernel-based algorithms build an estimator in $\mathcal
             H_{K,D}$. Thus, it is naturally to take $\mathcal H_{K,D}$ rather than $\mathcal H_K$ as the hypothesis space.

Kernel-based  boosting aims at learning an estimator
from $\mathcal H_{K,D}$ based on $D$. Using
different sets of weak learners, there are two strategies  of
kernel-based boosting. The  one is to employ functions in
$\mathcal H_{K,D}$ with small RKHS norms as the set of
weak learners. Using the standard gradient descent technique, this
type of  $L_2$-Boosting  boils down to iterative residual fitting
scheme and was   proved in \citep{Lin2019JMLR} to be almost
 over-fitting resistant in the sense that
its bias increases exponentially while its variance increases with a
exponentially small increment as the boosting iteration happens.
However, such a near over-fitting resistance is built upon some minimum eigen-value assumption of the kernel matrix, which is difficult to check for general data distributions and kernels.
The
other is to use $S$ as the set of weak learners \citep{Zhang2005}.
This strategy coincides with the well known greedy algorithms
\citep{Barron2008} and dominates in reducing the computational burden and deducing sparse estimator
\citep{Zhang2005}. However,
the learning rate of these algorithms are usually slow, especially,
only an order of $m^{-1/2}$ can be guaranteed \citep{Barron2008}.

In this paper, we focus on  designing a kernel-based  $L_2$-Boosting
algorithm  which  is near over-fitting resistant for general kernels and data distributions, user-friendly
and theoretically feasible. Our basic idea is to combine the classical truncation operator in \citep{Zhang2005} to reduce the variance  and a recently developed re-scaling technique in \citep{Wang2019} to accelerate the numerical convergence
rate. We thus name the new algorithm as kernel-based re-scaled boosting with truncation (KReBooT).
 Given
 a set of re-scaling parameters $\{\alpha_k\}_{k=1}^\infty$ with $\alpha_k
\in(0,1)$ and a set of non-decreasing step sizes
$\{l_k\}_{k=1}^\infty$. KReBooT starts with $f_{D,0}=0$
and then iteratively runs the following two steps:\\
 {\it Step 1 (Projection of gradient):} Find $g_k^*\in S$ such that
\begin{equation}\label{Step 1: gradient projection}
            g^*_k:=g^*_{k,D}:= \arg{\max _{g \in S}}|{\langle y-f_{D,k-1},g\rangle _m}|,
\end{equation}
where $\langle f,g\rangle_m:=\frac1m\sum_{i=1}^mf(x_i)g(x_i)$  and
$y$ is a function satisfying $y(x_i)=y_i$.\\
{\it Step 2 (Line search with re-scaling and truncation):} Define
\begin{equation}\label{Step 2:Line search}
            f_{D,k}:=(1-\alpha_k)f_{D,k-1}+\beta_k^*
            g^*_k,
\end{equation}
where
\begin{equation}\label{Minimum energy 1}
       \beta_k^*:=
       \arg\min_{\beta\in\Lambda_k}
       \left\|(1-\alpha_k)f_{D,k-1}+\beta
       g_k^*-y\right\|_m^2
\end{equation}
and $\Lambda_k:=[-\alpha_kl_k,\alpha_kl_k].$

Compared with the classical boosting algorithm in which the linear search is on $\mathbb R$ rather than $\Lambda_k$ and $\alpha_k=0$, KReBooT involves two crucial operators, i.e., re-scaling and truncation,  to control the structure of the derived estimator. In fact, we can derive the following structure constraint for KReBooT.

\begin{lemma}\label{Lemma:Bound for l1 norm}
Let $f_{D,k}$ be defined by (\ref{Step 2:Line search}). If
$\{l_k\}_{k=1}^\infty$ with $l_k\geq0$ is nondecreasing, then
$$
         \|f_{D,k}\|_{\ell_1}\leq l_k,\qquad\forall
         k=0,1,\dots.
$$
\end{lemma}

\begin{proof}
 It follows from $\Lambda_1=[-\alpha_1l_1,\alpha_1l_1]$,
$\alpha_1\leq 1$ and $f_{D,0}=0$ that
$$
         \|f_{D,1}\|_{\ell_1}\leq (1-\alpha_1)\|f_{D,0}\|_{\ell_1}+l_1 \leq l_1.
$$
If we assume $\|f_{D,k}\|_{\ell_1}\leq l_k$, then (\ref{Step 2:Line
search}) yields
\begin{eqnarray*}
             \|f_{D,k+1}\|_{\ell_1}
             &\leq&
             (1-\alpha_{k+1})\|f_{D,k}\|_{\ell_1}+\alpha_{k+1}l_{k+1}\\
             &\leq& (1-\alpha_{k+1})l_{k+1}+\alpha_{k+1}l_{k+1}=l_{k+1}.
\end{eqnarray*}
This proves Lemma \ref{Lemma:Bound for l1 norm} by
induction.
\end{proof}

Lemma \ref{Lemma:Bound for l1 norm} shows that the $\ell_1$ norm of the KReBoot estimator can be bounded by the step-size parameter $l_k$ via re-scaling and truncation. With this, we can tune $l_k$ to control the structure of the estimator and consequently derive a near over-fitting resistent learner. There are totally three types of
parameters in the new algorithm:   re-scaling parameter $\alpha_k$,   step-size parameter  $l_k$ and   iteration number $k$.  It should be highlighted that $l_k$ is imposed to control the structure, $\alpha_k$ is adopted to accelerate the numerical convergence rate and $k$ is the number of iteration. We will present detailed parameter-selection strategies after the theoretical analysis.

Although, there are more tunable parameters than the classical boosting algorithm, we will show that the difficulty of selecting each parameter is much less than other variants of boosting \citep{Zhang2005,Xu2017}. Furthermore, the re-scaling and truncation operators do not require additional computation in each boosting iteration. In fact,
for arbitrary $g\in S$ and $\beta\in\Lambda_k$, we have
\begin{eqnarray*}
    &&\left\|[(1-\alpha_k)f_{D,k-1}+\beta g]-y\right\|_m^2
     =
    \frac1m\sum_{i=1}^m[(1-\alpha_k)f_{D,k-1}(x_i)-y_i]^2+\beta^2\frac1m\sum_{i=1}^mg^2(x_i) \nonumber\\
    &-&2\beta\frac1m\sum_{i=1}^m[(1-\alpha_k)f_{D,k-1}(x_i)-y_i]g(x_i).
\end{eqnarray*}
Direct computation then yields
\begin{equation}\label{def.betak}
   \beta_k^*:=sign(\langle r_{k-1},g_k^*\rangle_m)
            \min\left\{\frac{\left|\langle
            r_{k-1},g^*_k\rangle_m\right|}{\|g_k^*\|_m^2},\alpha_kl_k\right\},
\end{equation}
where
  $r_{k-1}:=r_{D,k-1}:=y-(1-\alpha_k)f_{D,k-1}$ and $sign(\cdot)$ is the sign
function. With these, we summary the detailed implementation of KReBooT in Algorithm \ref{alg:1}.
\begin{algorithm}[t]
{\small
\begin{algorithmic}\caption{KReBooT}\label{alg:1}
\STATE {\bf Input}: $D=\{(x_i,y_i)\}_{i=1}^m$,  kernel $K(\cdot,\cdot)$.
\STATE {\bf Parameters:} Re-scaling parameter $\alpha_k\in(0,1)$, step size  $l_k\in\mathbb R_+$, and number of iterations $k=1,2,\dots,$
\smallskip
\FOR{$k=1,\ldots$}
\STATE{$\blacktriangleright$ (Projection of Gradient)}
 Find $g_k^*\in S$ satisfying (\ref{Step 1: gradient projection}).
\STATE{$\blacktriangleright$(Line search with re-scaling and truncation)}
Define
$$
            f_{D,k}:=(1-\alpha_k)f_{D,k-1}+\beta_k^*
            g^*_k,
$$
where $\beta_k^*$ is obtained by (\ref{def.betak}) and $f_{D,0}=0$.
\ENDFOR
\end{algorithmic}}
\end{algorithm}

\section{Numerical Convergence of KReBooT}\label{Sec.Properties}

Lemma \ref{Lemma:Bound for l1 norm} presents a structure constraint on the derived estimator of Algorithm \ref{alg:1}. In this section, we conduct the numerical convergence analysis for KReBooT. At first, we present a sufficient condition for $\{\alpha_k\}_{k=1}^\infty$ to guarantee the convergence of Algorithm \ref{alg:1} in the following theorem.

\begin{theorem}\label{Theorem: convergence}
Assume $|y_i|\leq M$ and $\kappa\leq 1.$  Given the non-decreasing sequence
$\{l_k\}_{k=1}^\infty$    and non-increasing
$\{\alpha_k\}_{k=1}^\infty$ with $\alpha_k\in (0,1)$, if
\begin{equation}\label{Condition-1}
    \lim_{k\rightarrow\infty}\alpha_k=0,\qquad\mbox{and}\quad \sum_{k=1}^\infty \alpha_k=\infty,
\end{equation}
then
\begin{equation}\label{Convergence 1}
  \lim_{k\rightarrow\infty}
      f_{D,k}=\left\{\begin{array}{cc}
      \arg\min_{f\in \mathcal H_{K,D}}
      \frac1m\sum_{i=1}^m(f(x_i)-y_i)^2, & \mbox{if}\  \lim_{k\rightarrow\infty}
     l_k=\infty,\\
      \arg\min_{f\in B_{L}}
      \frac1m\sum_{i=1}^m(f(x_i)-y_i)^2, & \mbox{if}\  l_k=L,
      \end{array}
      \right.
\end{equation}
where $B_L:=\left\{f=\sum_{i=1}^ma_iK_{x_i}:\|f\|_{\ell_1}=\sum_{i=1}^m|a_i|\leq L\right\}$.
\end{theorem}

Algorithm \ref{alg:1} shows that the range of linear search is $[-\alpha_kl_k,\alpha_kl_k]$, which means that  $\lim_{k\rightarrow\infty}\alpha_k=0$ together with not so large $l_k$ guarantee the convergence of KReBooT. An extreme case is to set $\alpha_k=0$, where the step size is always to be zero and the output is always $f_{D,0}=0$. Under this circumstance, the condition $\sum_{k=1}^\infty \alpha_k=\infty$ guarantees  the effectiveness  of the  boosting iteration and controls where the algorithm converges. For different $l_k$, KReBooT converges either to an empirical  kernel-based least-squares solution or a kernel-based LASSO solution. Therefore, KReBooT with $l_k=L$ is a feasible and efficient algorithm to solve the kernel-based LASSO, whose learning rates were established   in the learning theory community \citep{Shi2011,Shi2013,Guo2013}.

Due to the special iteration rule of KReBooT, its numerical convergence rate depends heavily on the re-scaling parameter $\alpha_k$. If $\alpha_k$ is too large, then the re-scaling operator offsets the effectiveness of the previous boosting iteration, making the numerical convergence rate be slow. On the contrary, if $\alpha_k$ is too small, then   step  sizes of the linear search are also very small, reducing the effectiveness of boosting iterations. Therefore, a suitable selection of the re-scaling parameter $\alpha_k$ is highly desired in KReBooT. In the following theorem, we show that, KReBooT with $\alpha_k=\frac2{k+2}$, can achieve the optimal numerical convergence rate of nonlinear approximation.

\begin{theorem}\label{Theorem:numerical convergence rate}
Assume $|y_i|\leq M$ and $\kappa\leq 1.$ For arbitrary
$h\in\mathcal H_{K,D}$ with $\|h\|_{\ell_1}<\infty$, if
$\alpha_k=\frac{2}{k+2}$ and $\{l_k\}_{k=1}^\infty$ is a sequence of
nondecreasing positive numbers satisfying
$\lim_{k\rightarrow\infty}l_k=\infty$, then
\begin{equation}\label{hyp.1.1}
              \|y-f_{D,k}\|_m^2- \|y-h\|_m^2\leq  32\max\left\{16k^*_h
              \left[M^2(k^*_h+4)+8 l_{k^*_h}^2\right],
      15\right\}
             k^{-1},
\end{equation}
where $k^*_h$ is the smallest positive integer  satisfying $
            l_{k^*_h}\geq \|h\|_{\ell_1}.
$
\end{theorem}

In (\ref{hyp.1.1}), the convergence rate depends on $k^*_h$ and $l_{k^*_h}$. For a given $h$ with $\|h\|_{\ell_1}<\infty$ and a nondecreasing positive numbers $\{l_k\}_{k=1}^\infty$ with $\lim_{k\rightarrow\infty}l_k=\infty$, there always exists a constant $k^*_h$ such that $l_{k^*_h}\geq \|h\|_{\ell_1}$. Under this circumstance,  $k^*_h$ and $l_{k^*_h}$ can be regarded as constants in the estimate. However, it should be mentioned that for $k$ satisfying $l_k< \|h\|_{\ell_1}$, the boosting iteration in Algorithm \ref{alg:1} is not effective and the  algorithm requires   increasing property of $\{l_k\}_{k=1}^\infty$. Once  $l_k\geq  \|h\|_{\ell_1}$ for some $k$, then the algorithm converges of an order $\mathcal O(1/k)$. In this way, the selection of $l_k$ is crucial. We recommend to set $l_k=c_0\log (k+1)$ for some $c_0>0$.

A main problem of the classical $L_2$-Boosting algorithm  is its low numerical convergence rate. Under the same setting as Theorem \ref{Theorem:numerical convergence rate}, it was shown in \citep{Livshits2009} that  the  order of numerical convergence rate of $L_2$-Boosting
lies in $(k^{-0.3796},k^{-0.364})$,
 which is much slower than
the minimax nonlinear approximation rate \citep{DeVore1996},  $\mathcal O(k^{-1})$, and leaves a large room to be improved. Furthermore, there lacks structure constraint for the derived boosting estimator, which requires in-stable relationship between generalization performance and boosting iterations. Noticing this,  \citep{Zhang2005} proposed RTboosting to control the structure for the derived estimator and then improve the generalization performance. However, the best numerical convergence rate of this variant is $\mathcal O(k^{-1/3})$ and the $\ell_1$ norm of the estimator satisfies $\|f_{D,k}\|_{\ell_1}=\mathcal O(k^{1/3})$.  Using the re-scaling technique in \citep{Bagirov2010}, \citep{Xu2017} proved that RBoosting can achieve the optimal numerical convergence rate as order $\mathcal O(k^{-1})$. The problem is, however, the $\ell_1$ norm of the derived estimator is much larger than $\mathcal O(k^{1/3})$, which makes the algorithms be sensitive to the re-scaling parameter and   number of iterations.
Theorem \ref{Theorem:numerical convergence rate} embodies the advantages of RBoosting in terms of optimal numerical convergence rate and RTboosting by means of providing controllable $\ell_1$ norm of the derived estimate.


\section{Learning Rate Analysis}\label{Sec.Learning}
In this section, we are interested in deriving fast learning rates for KReBooT.
Our analysis is carried out in the
framework of statistical learning theory \citep{Cucker2007} , where $D=\{z_i\}_{i=1}^m=\{(x_i,y_i)\}_{i=1}^m\subset \mathcal Z$ are assumed to be drawn independently
  according to an unknown joint distribution $\rho:=\rho(x,y)=\rho_X(x)\rho(y|x)$ with $\rho_X$  the marginal distribution and $\rho(y|x)$ the
conditional distribution. The learning performance of an estimator $f$ is measured by the
generalization error $ \mathcal E(f):=\int_{\mathcal Z}(f(x)-y)^2d\rho.$ Noting that the regression function defined by  $f_\rho(x)=E[y|X=x]$ minimizes the generalization error,  our target is then to learn a
function $f_D$ to approximate $f_\rho$ such that
\begin{equation}\label{equality}
                     \mathcal E(f_{D})-\mathcal E(f_\rho)=\|f_{D}-f_\rho\|^2_\rho
\end{equation}
 is as small as possible.

To quantify the learning performance of KReBooT, some priori information including the regularity of the regression function and capacity of the assumption space should be given at first.
Define  $L_K:
L_{\rho_X}^2\rightarrow L_{\rho_X}^2$ (or $\mathcal H_K\rightarrow\mathcal H_K$)  as
\begin{equation}\label{Integral operator}
    L_K f  :=\int_{\mathcal X} K_x f(x)d\rho_X.
\end{equation}
Since $K$ is positive-definite, $L_K$ is a positive operator. The following two assumptions describe the regularity of $f_\rho$ and capacity of $\mathcal H_K$, respectively.

 \begin{assumption}\label{Assumption:smoothness}
There exists an $h_\rho\in
         L_{\rho_X}^2$ such that
\begin{equation}\label{regularitycondition}
         f_\rho=L_K h_\rho=\int_{\mathcal X} K_x h_\rho(x)d\rho_X.
\end{equation}
\end{assumption}

Assumption \ref{Assumption:smoothness} describes the  regularity
of the regression function and is a bit stronger than the standard assumption $f_\rho\in\mathcal H_K$, i.e.  $f_\rho=L_K^{1/2} h_\rho$. Such an assumption is to guarantee that there exists an $f_0\in\mathcal H_{D,K}$  which approximates $f_\rho$ well with high probability and satisfies $\|f_0\|_{\ell_1}\leq C$ for some constant $C$ depending only on $h_\rho$ (See Lemma \ref{Lemma:l1norm for media} below). The aforementioned property is standard for boosting algorithms \citep{DeVore1996,Zhang2005,Temlyakov2008,Barron2008,Mukherjee2013,Temlyakov2015,Petrova2016,Xu2017,Wang2019}.
 The second assumption concerns
the eigenvalue-decay associated with $L_K$.
\begin{assumption}\label{Assumption:eigenvalue decay}
Let $\{(\mu_\ell, \phi_\ell)\}_{\ell=1}^\infty$ be a set of
normalized eigenpairs of $L_K$ with  $\{\mu_\ell\}_{\ell=1}^\infty$
arranging
  in a non-increasing order. For $0<s<1$ and some $c>0$,
we assume
\begin{equation}\label{eigenvalue value decaying}
     \mu_\ell\leq c\ell^{-1/s},\qquad \forall \ell\geq 1.
\end{equation}
\end{assumption}

The above assumption depicts the capacity of $\mathcal H_K$ as well as $\mathcal H_{K,D}$. Since the estimator derived by Algorithm \ref{alg:1} under Assumption \ref{Assumption:smoothness} is always in $\mathcal H_{K,D}$ with structure constraints, its generalization performance depends heavily on the kernel and consequently $s$ in Assumption \ref{Assumption:eigenvalue decay}. Assumption \ref{Assumption:eigenvalue decay} is   slight stronger than the effective dimension assumption in \citep{Guo2017,Lin2017,Lu2018} and is widely used in bounding learning rates for numerous kernel approaches \citep{Caponnetto2007,SteinwartHS,Raskutti2014,Zhang2015,Lin2019JMLR}.
By the help of the above two assumptions, we present our third main result in the following theorem, which quantifies the learning performance of KReBooT.

\begin{theorem}\label{Theorem:Bound in Probability}
Let $0<\delta<1$ and $f_{D,k}$ be defined by (\ref{Step 2:Line
search}) with $\alpha_k=\frac{2}{k+2}$ and $l_k=c_0\log (k+1)$ for some $c_0>0$. Under Assumption \ref{Assumption:smoothness} and Assumption \ref{Assumption:eigenvalue decay} with some $0<s<1$ and $c>0$,
if $|y_i|\leq M$, $\kappa\leq 1$ and   $m^{-1/(s+1)}\log^2\frac6\delta\leq 1$, then for
arbitrary $k\geq c_1m^{\frac1{1+s}}$  with confidence
$1-\delta$, there holds
\begin{equation}\label{oracle}
     \mathcal E(f_{D,k})-\mathcal
   E(f_\rho)
   \leq
    C\left(\log(k+1)\right)^2 m^{-\frac1{1+s}}
     \log^{4}\frac{18}\delta,
\end{equation}
where $c_1$ and $C$ are constants independent of $m$, $k$ or $\delta$.
\end{theorem}

Theorem \ref{Theorem:Bound in Probability} shows that for $k=m$ and sufficiently small $s$, KReBooT achieves a learning rate of order $\mathcal O(m^{-1}(\log m)^2)$. It should be mentioned that it is a new record for boosting-type algorithms. In particular,   under the similar setting as this paper, the learning rate of  RTboosting \citep{Zhang2005} is slower than $\mathcal O(m^{-1/2})$ while it of Rboosting \citep{Barron2008} is $\mathcal O(m^{-1/2})$. However, the derived learning rate in Theorem \ref{Theorem:Bound in Probability} is slower than some existing kernel approaches like kernel ridge regression \citep{Lin2017}, kernel gradient descent \citep{Lin2018a}, kernel partial least squares \citep{Lin2018b}, kernel conjugate descent \citep{Blanchard2016} and kernel spectral algorithms \citep{Guo2017}. The reason is that we impose the structure restrictions on the derived estimator as in Lemma \ref{Lemma:Bound for l1 norm} to reduce the testing time, enhance the interpretability and maintain the near overfitting resistant property of the algorithm. Under the similar  structure constraints, our derived learning rate is much faster than that of kernel-based LASSO \citep{Shi2011,Shi2013,Guo2013}.

To guarantee the good learning performance of KReBooT, there are two requirements on the number of iterations, i.e., $k\geq c_1m^{1/(1+s)}$ and $k$ is not exponential with respect to $m$. Noting Theorem \ref{Theorem:numerical convergence rate}, the former is necessary to derive an estimator of bias $m^{-1/(1+s)}$. The latter, benefiting from the structure constraint of KReBooT, shows its near over-fitting resistance, which is novel for kernel-based learning algorithms and essentially different from the results in \citep{Lin2019JMLR}, since we do not impose any lower bounds for eigenvalues of $L_K$.
This property shows that if $k$ is larger than a specific value of order  $m^{1/(1+s)}$, then running KReBooT does not bring essentially negative effect.

Noting that there are three tunable parameters in Algorithm \ref{alg:1}, that is, $\alpha_k$, $l_k$ and $k$. Our theoretical analysis and experimental verification below  show that KReBooT is stable
 with respect to $\alpha_k$ and it can be fixed to be $\frac2{k+2}$ before the learning process. Moreover, we can set $l_k=c_0\log (k+1)$ to guarantee the good structure of the KReBooT estimator. Here, $c_0$ is a parameter which affects the constant $C$ in (\ref{oracle}). In practice, it is somewhat important and should be specified by using some parameter-selection strategies such as ``hold-out'' \citep{Caponnetto2010} or cross-validation \citep{Gyorfi2002}. The selection of $k$ depends on $c_0$. If $c_0$ is extremely large, $\infty$ for example, then the truncation operator does not make sense and the performance of algorithm is sensitive to $k$. If $c_0$ is suitable, it follows from Theorem \ref{Theorem:Bound in Probability} that a large $k$, comparable with $m$ or larger, is good enough. Thus, in the practical implementation of KReBooT, we suggest to set $\alpha_k=\frac2{k+2}$, $k$ to be large and $l_k=c_0\log (k+1)$ with $c_0$ to be a tunable parameter. Under this circumstance, there is only one key parameter in the new algorithm, which is fewer than other variants of regularized boosting algorithms such as RTboosting and Rboosting.

\section{Experiments}\label{Sec.Experiments}
In this section, we shall conduct several simulations to verify the merits of the proposed boosting algorithm.  In all the  simulations, we consider the
following regression model:
\begin{equation}y_i=g(x_i)+\varepsilon_i, \quad i=1, 2, \cdots, N,\end{equation}
where $\varepsilon_i$ is the independent Gaussian noise, and
$$
   g(x):=h_2(\|x\|_2):=\left\{\begin{array}{cc}
   (1-\|x\|_2)^6(35\|x\|_2^2+18\|x\|_2+3), & 0<\|x\|_2\leq 1, x\in\mathbb{R}^3,\\
   0,&   \|x\|_2>1. \end{array}\right.
$$
The kernel used for the proposed KReBooT is chosen as $K(x, x')=h_3(\|x-x'\|_2)$ with
$$
     h_3(\|x\|_2):=\left\{\begin{array}{cc}
     (1-\|x\|_2)^4(4\|x\|_2^2+1), &  0<\|x\|_2\leq 1, x\in\mathbb{R}^3,\\
                                                                  0,&
      \|x\|_2>1. \end{array}\right.
$$
The reason why we make such choices  of $g(\cdot)$ and $K(\cdot)$ is to guarantee Assumption \ref{Assumption:smoothness} \citep{Chang2017}.
\begin{figure}[t]
\centering
\includegraphics[scale=0.343]{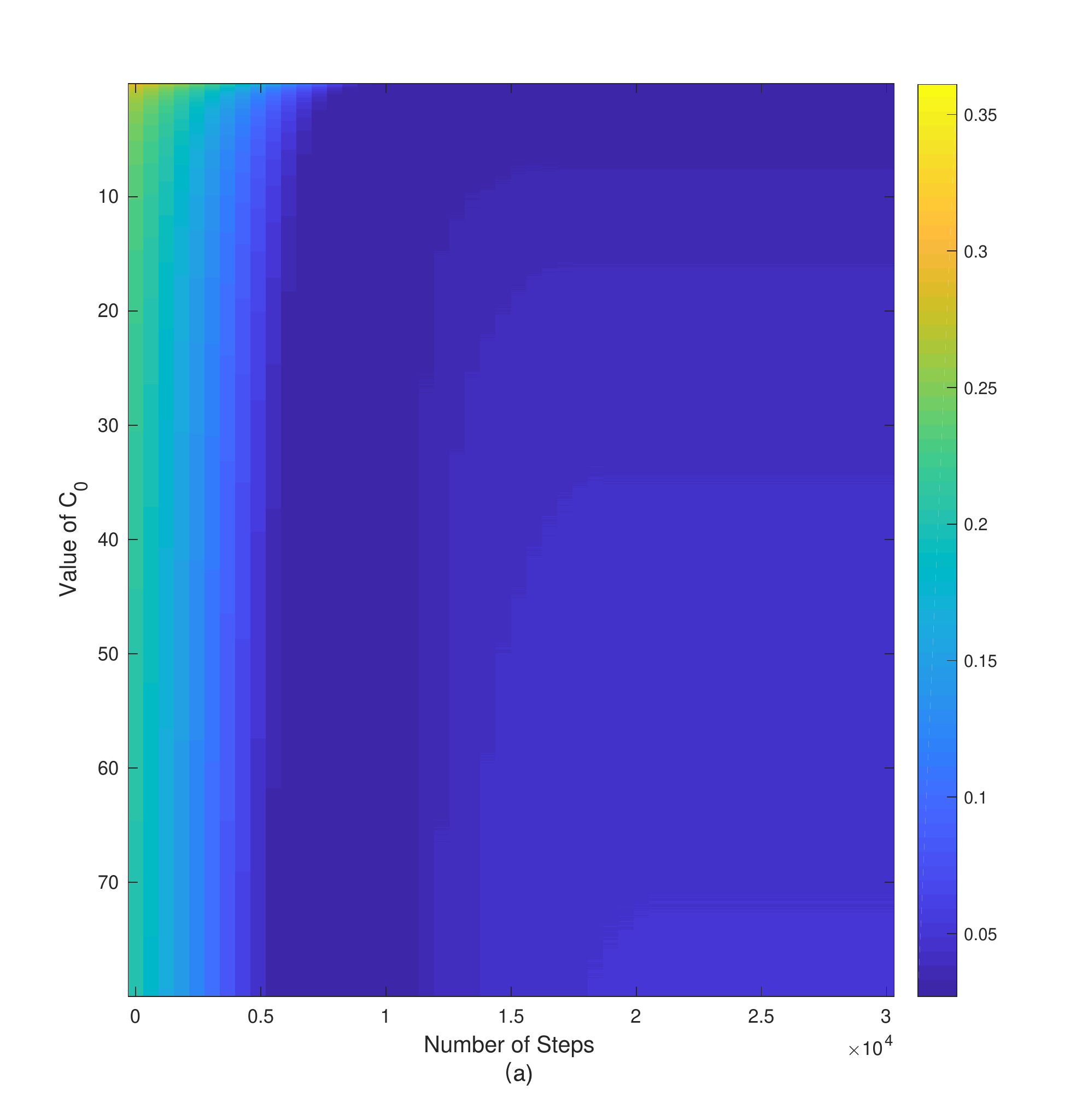}
\includegraphics[scale=0.343]{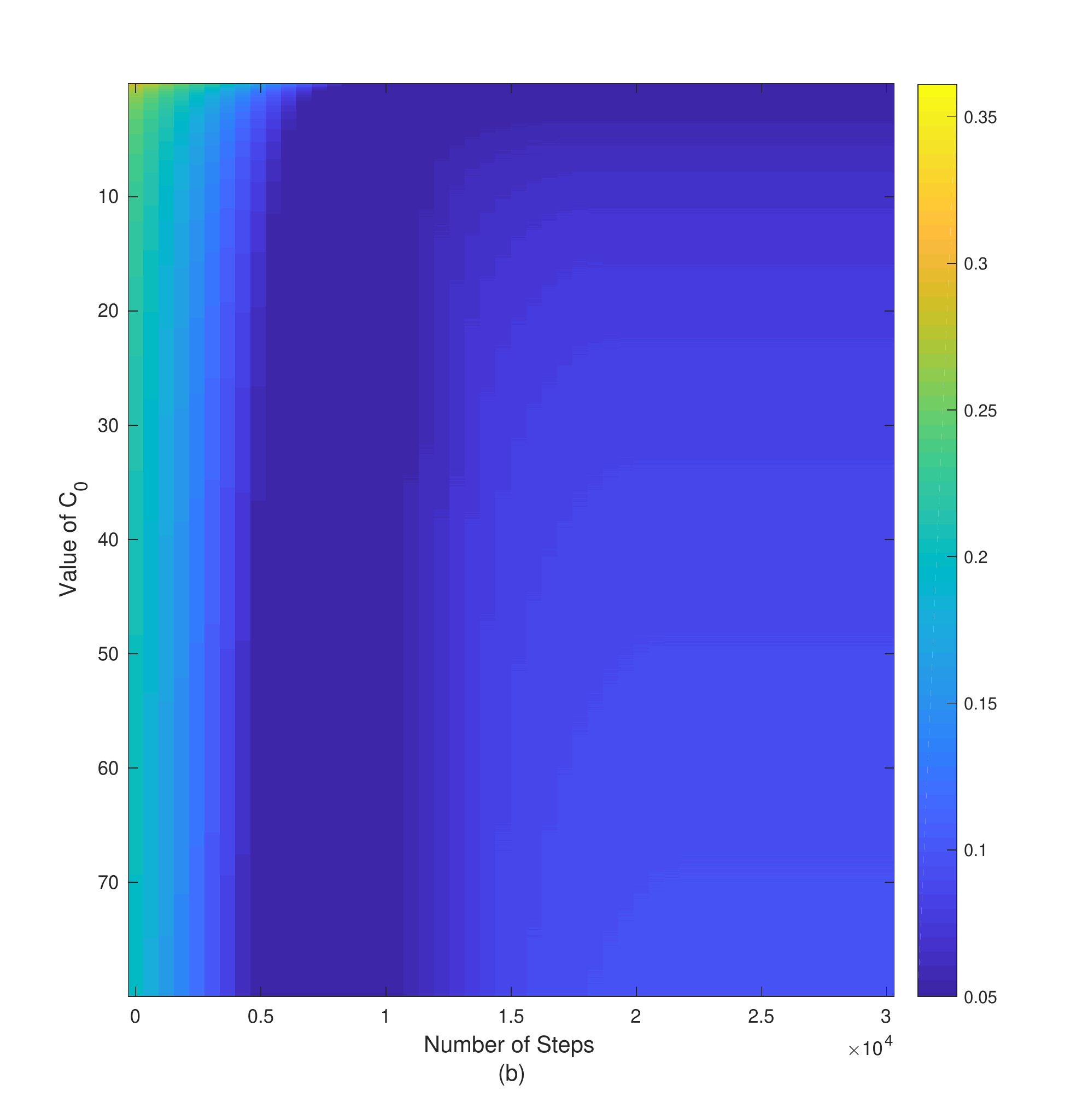}
\caption{The visualization of testing MSE of the proposed algorithm with varying the number
of iterations $k$ and the values of $c_0$. (a) $\varepsilon_i\sim\mathcal{N}(0,1)$; (b)
$\varepsilon_i\sim\mathcal{N}(0, 2)$.} \label{fig_phase}
\end{figure}

\textbf{Simulation I.}  Besides the number of iterations, there are  two additional parameters, the re-scaling parameter  $\alpha_k$ and the step-size parameter $l_k$, may play  important roles on the learning performance of KReBooT.  According to our theoretical assertions, if $\alpha_k=\frac{2}{k+2}$ and $l_k\sim c_0\log (k+1)$ for some $c_0>0$, then KReBooT can attain a fast  learning rate. Thus, this simulation mainly focuses on investigating the effect of the constant $c_0$ on the prediction performance of KReBooT. To this end,  we generate $m=500$ samples  for training and $m'=500$ samples
for testing, under two noise levels, i.e. $\varepsilon_i$ is i.i.d. drawn from either $\mathcal{N}(0, 1)$ or $\mathcal{N}(0, 2)$.
We then consider  50 candidates of $c_0$ that logarithmical  equally spaced in $[0.1,80]$.

Figure \ref{fig_phase} gives the visualization of testing mean-squared errors (MSE) of KReBooT with $\alpha_k=\frac{2}{k+2}$ via varying the number of iterations $k$ and the value of $c_0$. For any fixed step of iteration and $c_0$, the testing MSE is the average result over 20 independent trails.   It is easy to observe from this figure that, there exits a number of $c_0$'s in $[0.1, 10]$ for relatively small noise case or in $[0.1, 5]$ for relatively large noise case, such that the testing MSE  attains a stable value, neglecting the increasing   of iterations. This finding means that an appropriate choice of $c_0$ could avoid over-fitting. Therefore, for simplicity, we fix $c_0 =0.5$ in the following simulations.

\textbf{Simulation II.}  The objective of this simulation is to
describe the relation between the prediction accuracy and  the size of
training samples for the proposed KReBooT.  We thus generate $m=300, 900, 1500,
4500, 12000$ samples, respectively, for training, and $m'=500$ samples
for testing.  Similar to the previous simulation, we also consider two noise levels.
\begin{figure}[t]
\centering
\includegraphics[scale=0.336]{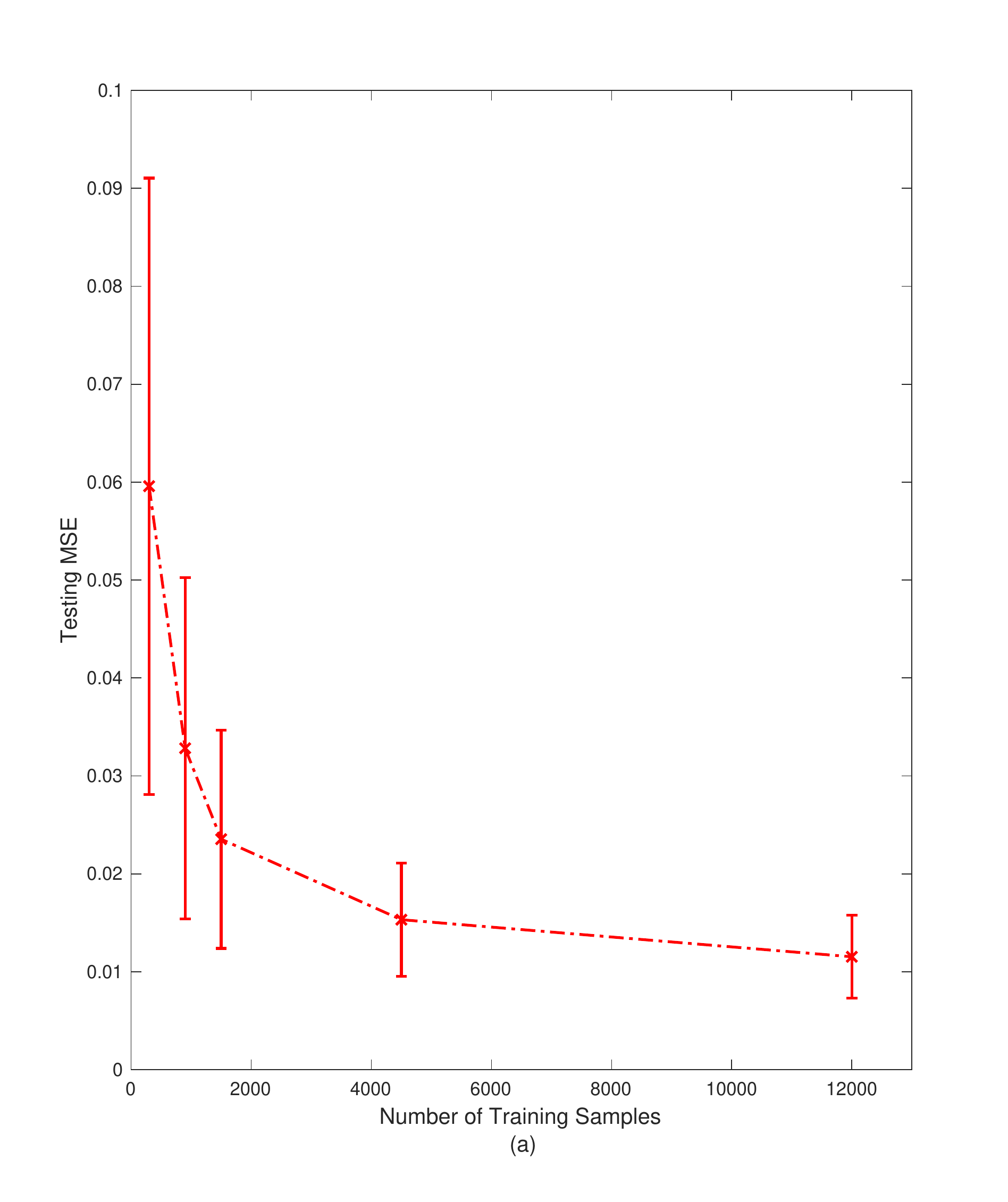}
\includegraphics[scale=0.336]{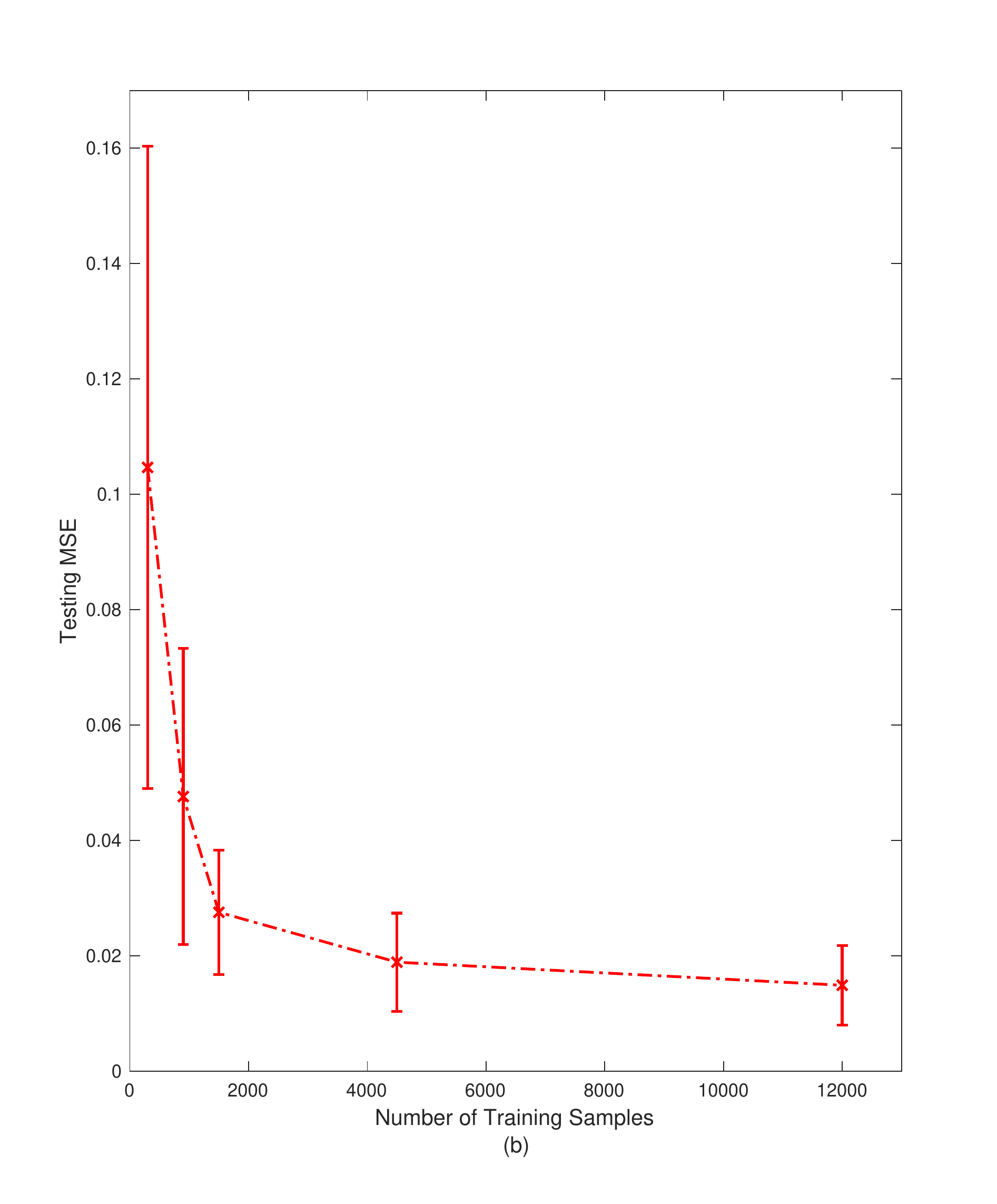}
\caption{The testing MSE versus the different sizes of training
samples. (a) $\varepsilon_i\sim\mathcal{N}(0,1)$; (b)
$\varepsilon_i\sim\mathcal{N}(0, 2)$.} \label{fig_rate}
\end{figure}

Figure \ref{fig_rate} depicts the average results over 100
independent trials. It is not hard to observe from this figure that the testing MSE deceases as the number of
training samples increases in two noise levels. This partially
supports the assertion presented in Theorem \ref{Theorem:Bound in Probability}.

\textbf{Simulation III.}  In this simulation, we shall compare the prediction performance of the proposed KReBooT with some other kernel-based methods, including kernel Lasso (Klasso)~\citep{Wang2007}, kernel ridge regression (KRR)~\citep{Caponnetto2007}, and three kernel version of popular boosting algorithms, i.e., $\epsilon$-boosting~\citep{Hastie2007}, rescale-boosting~\citep{Wang2019},  regularized boosting with truncation~\citep{Zhang2005}. We refer to these three kernel-based boosting algorithms as $\epsilon$-Kboosting,  KRboosting and KRTboosting, respectively. For two different noise levels, we firstly generate $m=300$ or $1000$ samples to built up the training set, and then generate a validation set of size 500 for tuning the parameters of different methods, and another 500 samples to evaluate the performances in terms of MSE.

Table 1 documents the average MSE over 100 independent runs. Numbers
in parentheses are the standard errors. It is not hard to see that the performance of the proposed KReBooT is comparable with Klasso and KRboosting, and clearly better than others.  Two important things should be further emphasized. Firstly, though KRboosting illustrates a similar good generalization capability as KReBooT, this algorithm is more likely to overfit, just as the following simulation shown.  Secondly, Klasso requires much more time in the training process than the proposed KReBooT.

\begin{table}[t]
\scriptsize \tabcolsep 0.034in
\renewcommand{\arraystretch}{1.27}
\caption{Prediction Performance of  Different Methods under
Different Settings} \label{table:1} \centering
\begin{tabular}{@{}*{9}{c}}
\hline
\multirow{2}{*}{{\shortstack{Training Size}}} & \multirow{2}{*}{{\shortstack{Noise Level}}}&\multicolumn{6}{c}{Methods} \\
\cline{4-9}
&&&KReBooT&KRboosting&KRTboosting&$\epsilon$-Kboosting&Klasso&KRR \\
\hline{}
\multirow{2}{*}{$m=300$} &\multirow{1}{*}{{\shortstack{$\sigma^2=1$}}}& &$0.066\pm0.036$ &$0.064\pm0.036$&$0.075\pm0.037$&$0.074\pm0.036$&$0.065\pm0.036$&$0.102\pm0.036$       \\

\cline{2-9}
&\multirow{1}{*}{{\shortstack{$\sigma^2=2$}}} & &$0.086\pm0.049$ &$0.086\pm0.047$ &$0.101\pm0.048$ &$0.100\pm0.049$&$0.088\pm0.048$&$0.137\pm0.047$      \\

\cline{1-9}
\multirow{2}{*}{$m=1000$} &\multirow{1}{*}{{\shortstack{$\sigma^2=1$}}}& & $0.027\pm0.013$ &$0.025\pm0.011$ &$0.030\pm0.010$&$0.029\pm0.010$&$0.025\pm0.009$&$0.043\pm0.013$      \\

\cline{2-9}
&\multirow{1}{*}{{\shortstack{$\sigma^2=2$}}} & &$0.041\pm0.021$ &$0.039\pm0.019$ &$0.049\pm0.022$ &$0.049\pm0.023$&$0.040\pm0.019$&$0.072\pm0.026$     \\

\cline{1-7}
\hline
\end{tabular}
\end{table}

\textbf{Simulation IV.}  In this simulation, we shall
show the overfitting resistence of KReBooT, as compared with its two cousins, i.e.,
KRboosting  and KRTboosting.  Here we generate 500 samples for
training, and another 500 samples for testing, with $\varepsilon_i$ is i.i.d. drawn from $\mathcal{N}(0, 1)$.  Figure
\ref{fig_overfitting}(a) clearly demonstrates the merits of our
proposed KReBooT, that is, the obtained testing MSE does't
increase as the number of iterations increases. In addition,
different from KRboosting and KRTboosting, the $\ell_1$ norm of the
coefficients obtained by KReBooT could converge to a fixed
value as shown in Figure \ref{fig_overfitting}(b),  which conforms to the assertion of structure constraints of KReBoot, just as  Lemma \ref{Lemma:Bound for l1 norm} purports to show.

\textbf{Simulation V.}  In this simulation, we mainly show the
 coefficients estimation behavior of KReBooT. Similar to the above simulation, KRboosting  and KRTboosting are considered for comparison. The aim now is to show the reason why
KReBooT is overfitting-resistant. Thus,  we generate 500 samples
for training, and another 500 samples for testing, under two
different noise levels, that is, $\varepsilon_i$ is iid drawn from
$\mathcal{N}(0, 1)$ and $\mathcal{N}(0, 2)$. Figures
\ref{fig_l1norm} exhibits the numerical results. It can be found in
both cases that the $\ell_1$ norm of the  KReBooT estimator increases much
slower than that of other algorithms, which implies that the variance of KReBooT keeps almost the same and is much smaller than other algorithms when the iterations increases.    Thus, the generalization error does not
increase very much as the iteration happens.

\begin{figure}[t]
\centering
\includegraphics[scale=0.335]{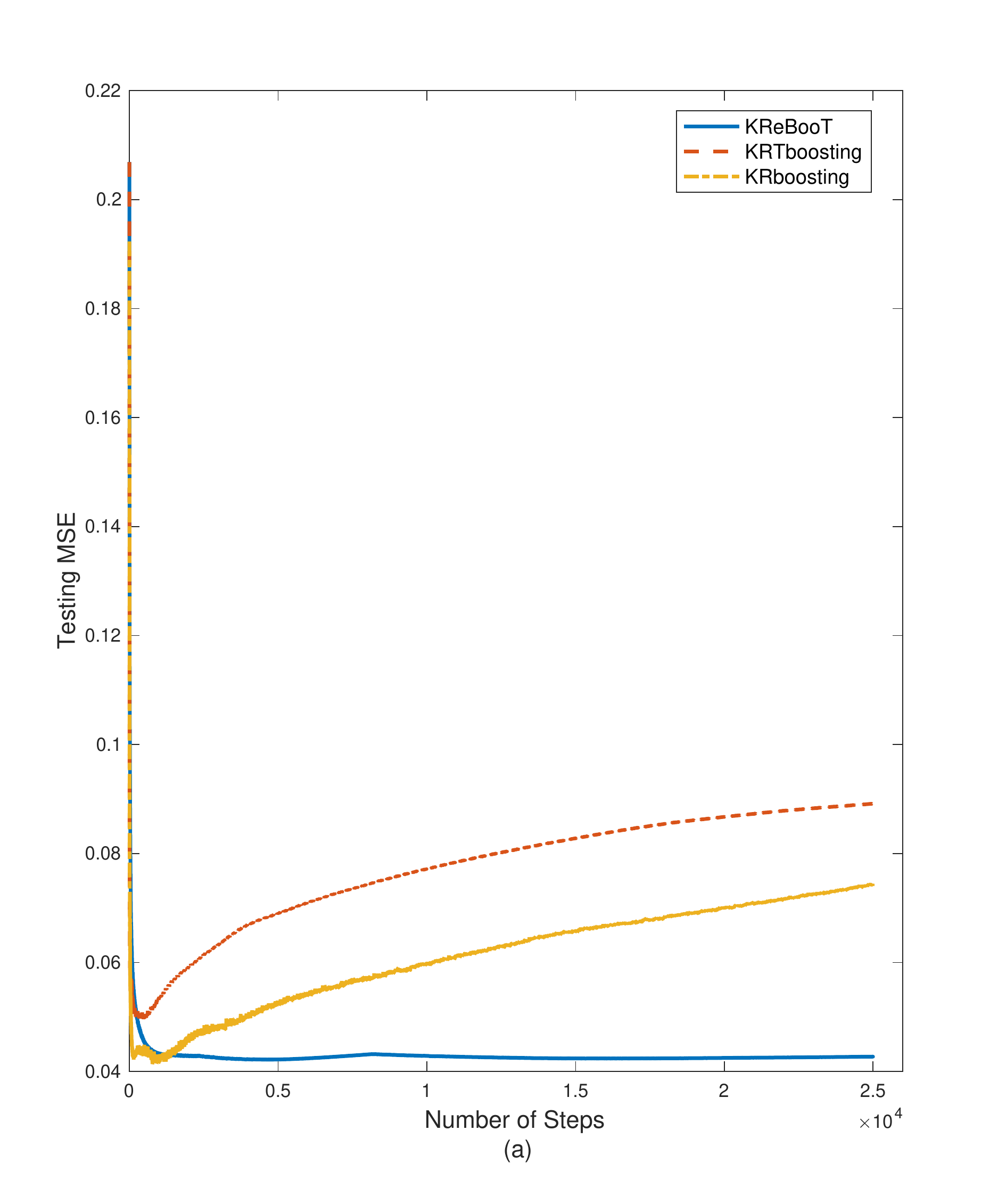}
\includegraphics[scale=0.335]{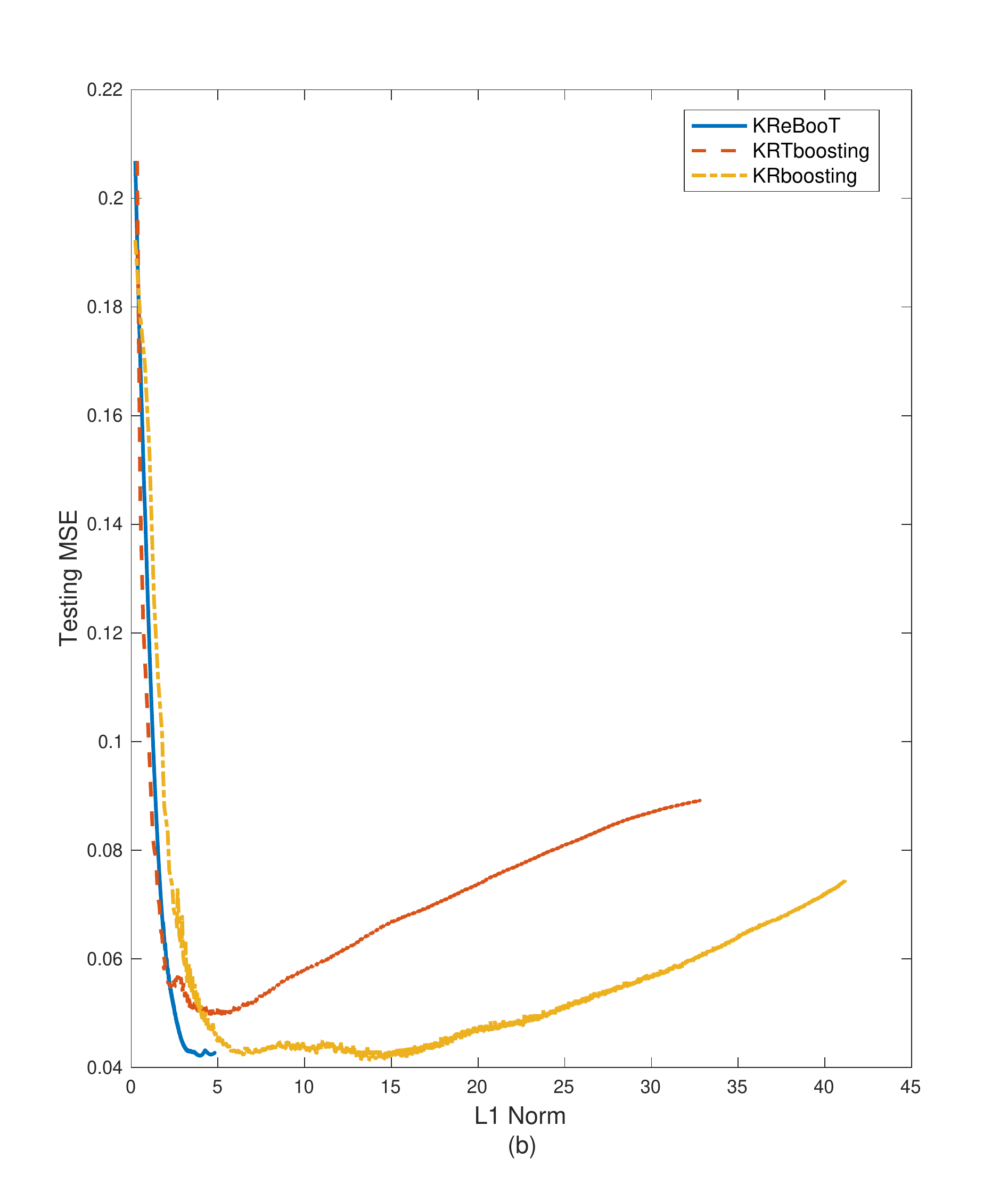}%
\caption{The prediction results obtained by three different
kernel-based boosting algorithms. (a) The testing MSE versus the number of
steps; (b) The testing MSE versus the $\ell_1$ norm.}
\label{fig_overfitting}
\end{figure}
\begin{figure}[h]
\centering
\includegraphics[scale=0.335]{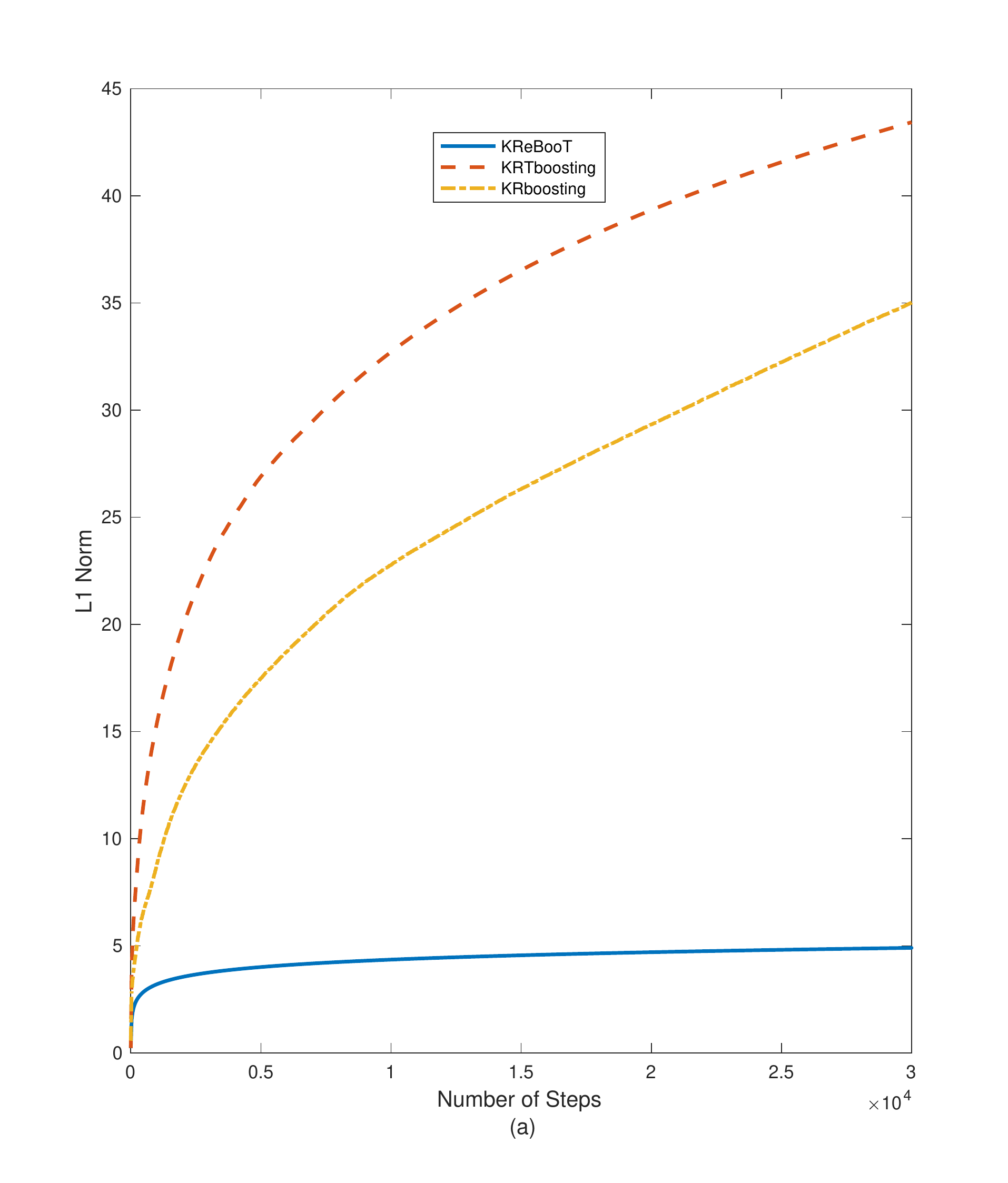}
\includegraphics[scale=0.335]{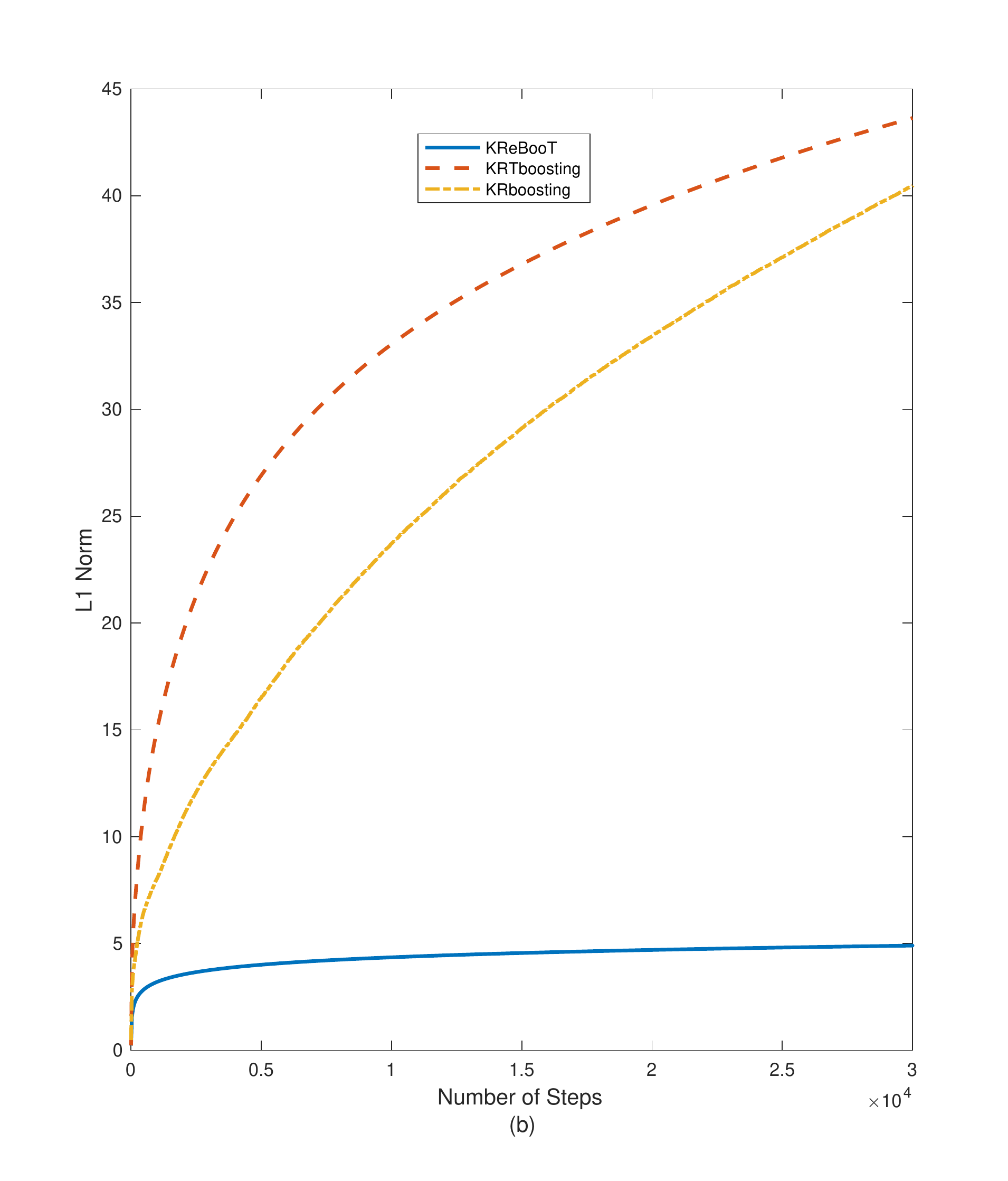}
\caption{The coefficients estimated by three different kernel-based boosting
algorithms versus the number of steps. (a)
$\varepsilon_i\sim\mathcal{N}(0,1)$; (b)
$\varepsilon_i\sim\mathcal{N}(0, 2)$.} \label{fig_l1norm}
\end{figure}

\section{Proofs}\label{Sec.Proof}

\subsection{Proof of Theorem \ref{Theorem: convergence}}

 Before presenting the proof of Theorem \ref{Theorem: convergence}, we at
first proving the following  lemma, which shows the role of iterations in kernel-based $L_2$-Boosting.

\begin{lemma}\label{lemma:Role of iteration}
Let  $\Lambda_k:=[-\alpha_kl_k,\alpha_kl_k]$. For arbitrary  $h\in
\mathcal H_{K,D}$ with $\|h\|_{\ell_1}<\infty$, if $|y_i|\leq M$, $\kappa\leq 1$ and $\{l_k\}$ is non-decreasing, then
\begin{equation}\label{numeral bridge}
       \|y-f_{D,k}\|_m^2-\|y-h\|^2_m\leq(1-\alpha_k)(\|y-f_{D,k-1}\|_m^2-\|y-h\|^2_m)+2\alpha_k^2(M+\|h\|_{\ell_1})^2
\end{equation}
holds for all $k\geq k^*_h:=\arg\min_{k}\{l_k\geq \|h\|_{\ell_1}\}$.
\end{lemma}

\begin{proof}
  For $h\in \mathcal H_{K,D}$, write
$h=\sum_{i=1}^ma_{i,h}K_{x_i}$ and
$\|h\|_{\ell_1}=\sum_{i=1}^m|a_{i,h}|$. Define $S^*=\{\pm
K_{x_i}:i=1,\dots,m\}$, then
$h=\sum_{i=1}^m|a_{i,h}|sign(a_{i,h})K_{x_i}$. For $k\geq k^*_h$,
set $\beta=\alpha_k\|h\|_{\ell_1}$ and notice $\beta\in \Lambda_k$. We get from (\ref{Minimum
energy 1}) that
$$
     \|y-f_{D,k}\|_m^2\leq \|y-(1-\alpha_k)f_{D,{k-1}}-\beta
      g_k^*\|_m^2=\|(1-\alpha_k)(y-f_{D,{k-1}})+\alpha_ky-\beta
      g_k^*\|_m^2.
$$
Since $|y|\leq M$, $|g_i|\leq\kappa\leq 1$ and
$\beta=\alpha_k\|h\|_{\ell_1}$, we have
\begin{eqnarray*}
   &&\|(1-\alpha_k)(y-f_{D,{k-1}})+\alpha_ky-\beta
      g_k^*\|_m^2\\
      &+&
      \|(1-\alpha_k)(y-f_{D,{k-1}})-\alpha_ky+\beta
      g_k^*\|_m^2\\
      &=&
      2(1-\alpha_k)^2\|y-f_{D,{k-1}}\|_m^2+2\|\alpha_ky-\beta
      g_k^*\|_m^2\\
      &\leq&
     2(1-\alpha_k)^2\|y-f_{D,{k-1}}\|_m^2+2\alpha_k^2(M+\|h\|_{\ell_1})^2.
\end{eqnarray*}
But (\ref{Step 1: gradient projection}) implies that  for arbitrary
$g\in S^*$,
\begin{eqnarray*}
   &&\|(1-\alpha_k)(y-f_{D,{k-1}})-\alpha_ky+\beta
      g_k^*\|_m^2\\
      &\geq&
      (1-\alpha_k)^2\|y-f_{D,{k-1}}\|_m^2-2\langle(1-\alpha_k)(y-f_{D,{k-1}}),\alpha_ky-\beta
      g_k^*\rangle_m\\
      &\geq&
      (1-\alpha_k)^2\|y-f_{D,{k-1}}\|_m^2-2\langle(1-\alpha_k)(y-f_{D,{k-1}}),\alpha_ky-\beta
      g\rangle_m.
\end{eqnarray*}
Then,
\begin{eqnarray*}
       &&\|(1-\alpha_k)(y-f_{D,{k-1}})+\alpha_ky-\beta
      g_k^*\|_m^2\\
       &\leq&
    2(1-\alpha_k)^2\|y-f_{D,{k-1}}\|_m^2+2\alpha_k^2(M+\|h\|_{\ell_1})^2\\
       &-&
      (1-\alpha_k)^2\|y-f_{D,{k-1}}\|_m^2+2\langle(1-\alpha_k)(y-f_{D,{k-1}}),\alpha_ky-\beta
      g\rangle_m.
\end{eqnarray*}
Since the above estimate holds for arbitrary $g_i\in S^*$,
$i=1,\dots,m$, it also holds for arbitrary convex combination of
$\{\pm K_{x_1}, \dots,\pm K_{x_m}\}$. In other words, the above
estimate holds for $\sum_{i=1}^mb_i sign(a_{i,h})K_{x_i} $ with
$b_i\geq0$ and $\sum_{i=1}^mb_i=1$. Setting
$b_i=\frac{a_{i,h}sign(a_{i,h})}{\sum_{i=1}^m|a_{i,h}|}$, it follows
from $\beta=\alpha_k\|h\|_{\ell_1}=\alpha_k\sum_{i=1}^m |a_{i,h}| $
that  for $k\geq k^*_h$,
\begin{eqnarray*}
       &&\|y-f_{D,k}\|_m^2\leq \|(1-\alpha_k)(y-f_{D,{k-1}})+\alpha_ky-\beta
      g_k^*\|_m^2\\
      &\leq&
    (1-\alpha_k)^2\|y-f_{D,{k-1}}\|_m^2+2\alpha_k^2(M+\|h\|_{\ell_1})^2\\
        &+&
        2\alpha_k\langle(1-\alpha_k)(y-f_{D,{k-1}}),y-h\rangle_m\\
      &\leq&
      (1-\alpha_k)^2\|y-f_{D,{k-1}}\|_m^2+2\alpha_k^2(M+\|h\|_{\ell_1})^2\\
      &+&
      \alpha_k(1-\alpha_k)\|y-f_{D,{k-1}}\|_m^2+\alpha_k\|y-h\|_m^2\\
      &=&
      (1-\alpha_k)\|y-f_{D,{k-1}}\|_m^2+\alpha_k\|y-h\|_m^2+2\alpha_k^2(M+\|h\|_{\ell_1})^2.
\end{eqnarray*}
Hence
$$
       \|y-f_{D,k}\|_m^2-\|y-h\|_m^2\leq(1-\alpha_k)(\|y-f_{D,k-1}\|_m^2-\|y-h\|^2_m)+2\alpha_k^2(M+\|h\|_{\ell_1})^2.
$$
  This completes the proof of Lemma
\ref{lemma:Role of iteration}.
\end{proof}

With the help of the above lemmas, we are in a position to prove Theorem \ref{Theorem: convergence}.

\noindent{\bf Proof of Theorem \ref{Theorem: convergence}.}
 We firs prove   (\ref{Convergence 1}), which will be
divided into the following two steps.

{\it Step 1:  Limit inferior.} Let $h_\infty:=\arg\min_{f\in
\mathcal H_{K,D}}\frac1m\sum_{i=1}^m(f(x_i)-y_i)^2.$ It follows from
the positive-definiteness of   $K$   that there is a set of real
numbers $\{a_{i,\infty}\}_{i=1}^m$ satisfying
$\|h_\infty\|_{\ell_1}=\sum_{i=1}^m|a_{i,\infty}|<\infty$ such that
$h_\infty=\sum_{i=1}^ma_{i,\infty} K_{x_i}$. Set
       $k^*_\infty:=\arg\min_{k}\{l_k\geq \|h_\infty\|_{\ell_1}\}.
$ Then it follows from Lemma \ref{lemma:Role of iteration} that for all $k\geq k^*_\infty$, there holds
\begin{equation}\label{numeral bridge 1}
       \|y-f_{D,k}\|_m^2-\|y-h_\infty\|^2_m
       \leq(1-\alpha_k)(\|y-f_{D,k-1}\|_m^2-\|y-h_\infty\|^2_m)+2\alpha_k^2(M+\|h_\infty\|_{\ell_1})^2.
\end{equation}
We then use (\ref{numeral bridge 1}) to prove
\begin{equation}\label{low lim 1}
     {\lim\inf}_{k\rightarrow\infty}
     \|y-f_{D,k}\|_m^2=\|y-h_\infty\|^2_m
\end{equation}
by contradiction. Denote
$A_{\infty,k}:=\|y-f_{D,k}\|_m^2-\|y-h_\infty\|^2_m$. If (\ref{low
lim 1}) does not hold, then there exist $K_\infty\in\mathbb N$ and
$\gamma_\infty>0$ such that $A_{\infty,k}\geq \gamma_\infty$ holds
for all $k\geq K_\infty$. Since $\alpha_k\rightarrow0$, there exists
a  $K_\infty^*\in\mathbb N$ such that
$\alpha_k(M^2+\|h_\infty\|_{\ell_1})\frac{2}{\gamma_\infty}\leq\frac12$
for all $k\geq K_\infty^*$. Hence, it follows from (\ref{numeral
bridge 1}) that for arbitrary
$k>\max\{K_\infty,K_\infty^*,k^*_\infty\}$,
\begin{eqnarray*}
      &&A_{\infty,k}
       \leq
       (1-\alpha_k)A_{\infty,k-1}+2\alpha_k^2(M+\|h_\infty\|_{\ell_1})^2\\
      &\leq&
      A_{\infty,k-1}\left[1-\alpha_k+\alpha_k^2(M+\|h_\infty\|_{\ell_1})^2\frac{2}{\gamma_\infty}\right]
        \leq
        A_{\infty,k-1}(1-\alpha_k/2).
\end{eqnarray*}
This together with the assumption $\sum_{k=1}^\infty
\alpha_k=\infty$ yields $A_{\infty,k}\rightarrow0$ as
$k\rightarrow\infty$. Thus, the assumption $\gamma_\infty>0$ is
false and (\ref{low lim 1}) holds.

{\it Step 2:  Limit superior.} We then aim at deriving
\begin{equation}\label{super lim 1}
     {\lim\sup}_{k\rightarrow\infty}
     \|y-f_{D,k}\|_m^2=\|y-h_\infty\|^2_m.
\end{equation}
For arbitrary $\nu>0$ and $k\geq k^*_\infty$, Lemma
\ref{lemma:Role of iteration} implies
\begin{equation}\label{add1}
      A_{\infty,k}-\nu
       \leq(1-\alpha_k)(A_{\infty,k-1}-\nu)+2\alpha_k^2(M+\|h_\infty\|_{\ell_1})^2.
\end{equation}
Define $B_{\infty,k}:=A_{\infty,k}-\nu$ and
$$
        U_\infty:=\{k:k\geq k^*_\infty,
        2\alpha_k(M +\|h_\infty\|_{\ell_1})^2\leq B_{\infty,k-1}\}.
$$
If $U_\infty$ is the finite or empty set, then it follows from $\alpha_k\rightarrow0$
that
$$
    {\lim\sup}_{k\rightarrow\infty}B_{\infty,k-1}\leq
    2(M +\|h_\infty\|_{\ell_1})^2\lim_{k\rightarrow\infty}\alpha_k=0.
$$
This implies
\begin{equation}\label{p.convergence 1}
    {\lim\sup}_{k\rightarrow\infty}A_{\infty,k}\leq\nu.
\end{equation}
If $U_\infty$ is infinite, we have from (\ref{add1}) that
\begin{equation}\label{numeral bridge 1.1}
      B_{\infty,k}
       \leq \left\{\begin{array}{cc}
       2\alpha_k^2(M+\|h_\infty\|_{\ell_1})^2, &B_{\infty,k-1}=0\\
       B_{\infty,k-1}\left(1-\alpha_k+2\alpha_k^2\frac{  (M+\|h_\infty\|_{\ell_1})^2
       }{B_{\infty,k-1}}\right),&
       B_{\infty,k-1}\neq 0,
       \end{array}\right.
\end{equation}
which implies
\begin{equation}\label{add2}
    B_{\infty,k}\leq \max\{
     2\alpha_k^2(M+\|h_\infty\|_{\ell_1})^2,B_{\infty,k-1}\},\qquad
\forall\ k\in U_\infty.
\end{equation}
 Furthermore, (\ref{low lim 1}) shows that there is
a subsequence $\{k_j\}$ such that $B_{\infty,{k_j}}\leq 0$,
$j=1,\dots.$ This means
$$
        U_\infty=\bigcup_{j=k_\infty^*}^\infty[m_j,n_j]
$$
for some $\{n_j\},\{m_j\}\subset\mathbb N$ satisfying $n_{j-1}<m_j-1.$ For $k\notin U_\infty$ and
$k\geq k_\infty^*$, we have
\begin{equation}\label{p.c.1}
    B_{\infty,k-1}< 2\alpha_k(M+\|h_\infty\|_{\ell_1})^2.
\end{equation}
For $k\in [m_j,n_j]$,  it follows from $\alpha_{k+1}\leq\alpha_k$, (\ref{add2}), (\ref{numeral bridge 1.1}) and (\ref{p.c.1}) that
\begin{eqnarray}\label{p.c.2}
     &&B_{\infty,k}\leq \max\{ 2\alpha_{m_j-1}^2(M+\|h_\infty\|_{\ell_1})^2,B_{\infty,m_{j-1}}\}\nonumber\\
     &\leq&
   \max\{2\alpha^2_{m_j-1}(M +\|h_\infty\|_{\ell_1})^2,(1-\alpha_{m_j-1})B_{\infty,m_{j-2}}+2\alpha^2_{m_j-1}(M +\|h_\infty\|_{\ell_1})^2\}\nonumber\\
   &\leq&
   (1-\alpha_{m_j-1}) 2\alpha_{m_j-1} (M +\|h_\infty\|_{\ell_1})^2 +2\alpha^2_{m_j-1}(M +\|h_\infty\|_{\ell_1})^2 \\
   &\leq&
   2\alpha_{m_j-1}(M +\|h_\infty\|_{\ell_1})^2.
\end{eqnarray}
Combining (\ref{p.c.1}) with (\ref{p.c.2}), we get
$$
     \lim\sup_{k\rightarrow\infty}B_{\infty,k}\leq0,
$$
which implies
$$
        \lim\sup_{k\rightarrow\infty}A_{\infty,k}\leq\nu.
$$
Since $\nu$ is arbitrary, (\ref {super lim 1}) holds. Thus,
(\ref{low lim 1}) and (\ref{super lim 1}) yield (\ref{Convergence
1}) by taking the uniqueness  of the solution to
$$
      {\arg\min}_{f\in\mathcal
      H_{K,D}}\frac1m\sum_{i=1}^m(f(x_i)-y_i)^2.
$$

If $h\in B_L$, it follows from Lemma \ref{lemma:Role of iteration} that (\ref{numeral bridge}) holds for all $k\in\mathbb
N$. Then   the same approach as above and Lemma
\ref{Lemma:Bound for l1 norm} implies  (\ref{Convergence 1}). This completes the proof of Theorem \ref{Theorem:
convergence}.
\hfill\BlackBox

\subsection{Proof of Theorem \ref{Theorem:numerical convergence rate}}

 To prove Theorem \ref{Theorem:numerical convergence rate}, we need the following lemma.

\begin{lemma}\label{Lemma:NUMBER THEORY}
Let $j_0>2$ be a natural number, $c_1<c_2\leq j_0$ and $\mathcal
C_0(\cdot)$ is a nondecreasing function defined on $\mathbb R_+$. If
$\{a_v\}_{v=1}^\infty$ satisfies
\begin{equation}\label{cond1 number theory}
   \left\{\begin{array}{cc}
    a_v\leq \mathcal C_0(v) v^{-c_1},& 1\leq v\leq j_0,\\
    a_v\leq a_{v-1}+\mathcal C_0(v-1)(v-1)^{-c_1}, & v> j_0
    \end{array}\right.
\end{equation}
and
\begin{equation}\label{cond2 number theory}
      a_{v}>\mathcal C_0(v){v}^{-c_1}\quad\mbox{implies}\quad a_{v+1}\leq
      a_{v}(1-c_2/{v}),\qquad \forall v>j_0,
\end{equation}
then there holds
$$
            a_k\leq   2^{\frac{c_1c_2+c_2}{c_2-c_1}} \mathcal
            C_0(k)k^{-c_1}, \qquad \forall k=1,2,\dots.
$$
\end{lemma}

\begin{proof}
For $1\leq v\leq j_0$, (\ref{cond1 number theory}) shows
$$
             a_v\leq \mathcal C_0(v)v^{-c_1}.
$$
Thus,
$$
      V=\{v\in\mathbb N:a_v> \mathcal C_0(v)v^{-c_1}\}
$$
does not contain $1,2,\dots,j_0$. Let $v\geq j_0+1$ satisfy
$v-1\notin V$ and $v\in V$, and $\eta$ be the largest positive
integer such that $[v,v+\eta]\subset V$. Then it follows from the
non-decreasing of $\mathcal C_0(\cdot)$ that
\begin{equation}\label{Number p 1}
      a_{v-1}\leq \mathcal C_0(v-1)(v-1)^{-c_1} \leq    \mathcal C_0(v)(v-1)^{-c_1},
\end{equation}
\begin{equation}\label{Number p 1.1}
      a_{v+j}>\mathcal C_0(v+j)(v+j)^{-c_1}\geq\mathcal C_0(v)(v+j)^{-c_1},
          \qquad \forall j=0,1,\dots,\eta.
\end{equation}
Hence, it follows from    (\ref{cond2 number theory}) and
(\ref{cond1 number theory}) that
\begin{eqnarray}\label{add3}
           a_{v+\mu}
            \leq
           a_v\Pi_{u=v}^{v+\mu-1}(1-c_2/u)
            \leq
           (a_{v-1}+\mathcal C_0(v-1)(v-1)^{-c_1})\Pi_{u=v}^{v+\mu-1}(1-c_2/u).
\end{eqnarray}
Inserting (\ref{Number p 1.1}) into the above estimate and noting
(\ref{Number p 1}), the nondecreasing of $\mathcal C_0(\cdot)$
implies
$$
        (v+\mu)^{-c_1}\leq\frac{a_{v+\mu}}{\mathcal C_0(v)}\leq2(v-1)^{-c_1}
        \Pi_{u=v}^{v+\mu-1}(1-c_2/u).
$$
Taking the logarithmic operator on both sides and using the
inequalities
$$
      \sum_{u=v}^{\ell-1}u^{-1}\geq\int_{v}^\ell t^{-1}dt=\ln (\ell/v), \qquad\mbox{and}\qquad \ln (1-t)\leq -t, \qquad
      t\in[0,1),
$$
we  derive
\begin{eqnarray*}
       -c_1\ln\frac{v+\eta}{v-1}
      \leq
      \ln
      2+\sum_{u=v}^{v+\eta-1}\ln(1-c_2/u)
       \leq
      \ln2-\sum_{u=v}^{v+\eta-1}c_2/u\leq \ln2-c_2\ln\frac{v+\eta}{v}.
\end{eqnarray*}
That is,
$$
        (c_2-c_1)\ln(v+\eta)\leq \ln2+(c_2-c_1)\ln
        v+c_1\ln\frac{v}{v-1}\leq (c_1+1)\ln2+(c_2-c_1)\ln
        v,
$$
where we used $\frac{v}{v-1}\leq 2$ for $v\geq j_0+1\geq2$ in the
last inequality. Then, for any segment $[v,v+\mu]\subset V$, there
holds
\begin{equation}\label{number proof1}
           v+\mu\leq 2^{(c_1+1)/(c_2-c_1)}v.
\end{equation}
For any $k\in\mathbb N$, if $k\notin V$, we have the desired
inequality in Lemma \ref{Lemma:NUMBER THEORY}. Assume $k\in V$ and
let $[v,v+\mu]$ be the maximal segment in $V$ containing $k$, then
it follows  from (\ref{add3}) and (\ref{Number p 1})  that
\begin{eqnarray*}
         a_k
         &\leq&
         a_{v-1}+\mathcal C(v-1)(v-1)^{-c_1}\Pi_{u=v}^{k-1}(1-c_2/u)
         \leq a_{v-1}+\mathcal C_0(v)(v-1)^{-c_1}\\
         &\leq&
        2\mathcal C_0(v)(v-1)^{-c_1}
          \leq
        2\mathcal C_0(k)k^{-c_1}\left(\frac{v-1}{k}\right)^{-c_1}.
\end{eqnarray*}
But (\ref{number proof1}) follows
$$
       \frac{k}{v-1}\leq2\frac{v+\mu}{v}\leq 2^\frac{c_2+1}{c_2-c_1}.
$$
Thus,
$$
       a_k\leq 2\mathcal C_0(k)k^{-c_1}2^\frac{c_1c_2+c_1}{c_2-c_1},
$$
which completes the proof of Lemma \ref{Lemma:NUMBER THEORY}.
\end{proof}

\noindent{\bf Proof of Theorem \ref{Theorem:numerical convergence rate}.}   Denoting
$$
       A_k=
        \|y-f_{D,k}\|_m^2- \|y-h\|_m^2,
$$
 it  follows from Lemma  \ref{lemma:Role of iteration} with $\alpha_k=\frac2{k+2}$ yields
\begin{equation}\label{iterative inequality}
      A_k\leq A_{k-1}-\frac2{k+2} A_{k-1}+\frac{8}{(k+2)^2}
        (M+\|h\|_{\ell_1})^2,\qquad\forall
      k\geq k^*_h.
\end{equation}
We then use Lemma \ref{Lemma:NUMBER THEORY} and (\ref{iterative
inequality}) to prove (\ref{hyp.1.1}). Let
$$
   \mathcal
      C_1:=\max\left\{16k^*_h(M^2k^*_h+4M^2+8l_{k^*_h}^2), 15\right\}
$$
 Due to (\ref{Minimum energy 1}), $|y_i|\leq M$,  $f_{D,0}=0$,   and $\alpha_k\leq 1$,
we have for arbitrary $k=1,2,\dots,$
 that
$$
      \|y-f_{D,k}\|_m^2\leq\|y-(1-\alpha_k)f_{D,k-1}\|_m^2
      \leq2  \|y\|_m^2+
      2(1-\alpha_k)^2\| y-f_{D,k-1} \|_m^2\leq 2(k+1)M^2.
$$
Hence
$$
      \|y-f_{D,k}\|_m^2-\|y-h\|_m^2
      \leq
      2M^2(k+2)+2\|h\|_{\ell_1}^2.
$$
Therefore, we have
$$
        A_v\leq \mathcal C_1 v^{-1},\qquad 1\leq v\leq k^*_h-1.
$$
Furthermore, (\ref{iterative inequality}) follows
$$
       A_v\leq A_{v-1}+\mathcal C_1 (v-1)^{-1}, \qquad v\geq k^*_h.
$$
If $A_v\geq\mathcal C_1v^{-1}$, it then follows from (\ref{iterative
inequality}) again that for arbitrary $v\geq k^*$
\begin{eqnarray*}
      A_{v+1}
      &\leq&
       A_{v}-\frac2{v+3} A_{v}+\frac{8A_v}{(v+3)^2A_v}
       (M +\|h\|_{\ell_1})^2\\
       &\leq&
       A_v\left(1-\frac2{v+3}+\frac{8v}{\mathcal C (v+3)^2}(M
       +\|h\|_{\ell_1})^2\right)
       \leq
       A_v\left(1-\frac{1.5}v\right).
\end{eqnarray*}
Then,   Lemma \ref{Lemma:NUMBER THEORY} with $\mathcal
C_0(\cdot)=\mathcal C_1$, $c_1=1$ and  $c_2=1.5$ shows that
$$
           A_k\leq   32\max\left\{16k^*(M^2k^*_h+4M^2+8\|h\|_{\ell_1}^2),
           15\right\}
            k^{-1},
$$
which completes the proof of Theorem \ref{Theorem:numerical convergence rate}. \hfill\BlackBox

\subsection{Proof of Theorem \ref{Theorem:Bound in Probability}}
We divide the proof of Theorem \ref{Theorem:Bound in Probability} into four steps: error decomposition and approximation error estimate, hypothesis error estimate, sample error estimate and generalization error analysis.

\subsubsection{Error decomposition and approximation error estimate}
 For arbitrary $\lambda>0$
   define
\begin{equation}\label{Def.three functions}
         f_{\lambda}^0=(L_K+\lambda I)^{-1}f_\rho, \
         f_\lambda=(L_K+\lambda I)^{-1}L_Kf_\rho,\
         f_{D,\lambda}^0=L_{K,D}(L_{K}+\lambda I)^{-1}f_\rho,
\end{equation}
where $L_{K,D}:\mathcal H_K\rightarrow\mathcal H_K$ is the empirical
operator  defined by
$$
         L_{K,D}f:=\frac1{m}\sum_{i=1}^mf(x_i)K_{x_i}.
$$
Denoting
\begin{eqnarray}
       \mathcal D(\lambda)&:=&  \mathcal E(f_{\lambda})-\mathcal
       E(f_\rho),\label{Def.Approximation error}\\
       \mathcal H(D,\lambda,k)&:=& \{\mathcal E(f^0_{D,\lambda})-\mathcal E(f_\lambda)
     +\mathcal E_D(f_{D,k})
      -\mathcal E_D(f_{D,\lambda}^0)\},\label{Def.Hypothesis error}\\
       \mathcal S_1(D,k) &:=&
        [\mathcal E(f_{D,k})-\mathcal E(f_\rho)]-
        [\mathcal E_D(f_{D,k})-\mathcal E_D(f_\rho)], \label{Def.Sample error1}\\
       \mathcal S_2(D,\lambda)&:=&
       [\mathcal E_D(f^0_{D,\lambda})-\mathcal E_D(f_\rho)]-[\mathcal
       E(f_{D,\lambda}^0)-\mathcal E(f_\rho)]  \label{Def.Sample error2}
\end{eqnarray}
with  $\mathcal E_D(f):=\frac1m(f(x_i)-y_i)^2$, we have
\begin{equation}\label{Error decomposition for l1}
   \mathcal E(f_{D,k})-\mathcal
   E(f_\rho)=
   \mathcal S_1(D,k)+\mathcal S_2(D,\lambda)+\mathcal H(D,\lambda,k)+\mathcal D(\lambda).
\end{equation}
Here, $\mathcal D(\lambda)$, $\mathcal S_1(D,k)+\mathcal S_2(D,\lambda)$ and $\mathcal H(D,\lambda,k)$ are the approximation error, sample error and hypothesis error, respectively.
Due to (\ref{equality}),
it can be found in \citep{Caponnetto2007,Lin2017} that Assumption \ref{Assumption:smoothness}  implies
\begin{equation}\label{Appr}
     \mathcal D(\lambda)\leq
     \lambda^2\|h_\rho\|^2_{L_{\rho_X}^2}.
\end{equation}

\subsubsection{Hypothesis error estimate}
To estimate the hypothesis space $\mathcal H(D,\lambda,k)$, we bound
  $ \mathcal
E(f^0_{D,\lambda})-\mathcal
       E(f_\lambda)$
and $\mathcal E_D(f_{D,k})
      -\mathcal E_D(f_{D,\lambda}^0)$, respectively. We adopted the recently
developed integral approaches in \citep{Lin2017,Guo2017}
  to derive an upper bound of $\mathcal
E(f^0_{D,\lambda})-\mathcal
       E(f_\lambda)$. The following lemma can be found in \citep{Lin2017}.

\begin{lemma}\label{Lemma:operator inequality}
Let $0<\delta<1$. If $\kappa\leq 1$, then
with confidence at least $1-\delta$, there holds
\begin{equation}\label{normprepareeta2}
      \left\|\left(L_{K} +\lambda I\right)^{-1/2}  (L_K- L_{K,
      D})\right\| \leq \frac{2}{\sqrt{m}}
     \left\{\frac{1}{\sqrt{m\lambda}} +\sqrt{{\mathcal
     N}(\lambda)}\right\} \log
      \frac2\delta,
\end{equation}
where  $
   {\mathcal N}(\lambda) = \hbox{Tr}\left((L_K + \lambda I)^{-1} L_K\right)
$ is  the trace of the operator $(L_K + \lambda I)^{-1} L_K$.
\end{lemma}

Based on Lemma \ref{Lemma:operator inequality}, we deduce the
following lemma.

\begin{lemma}\label{Lemma:Hypo 1.1}
Let $0<\delta<1$. If  $\kappa\leq 1$, Assumptions
\ref{Assumption:smoothness} and \ref{Assumption:eigenvalue decay}
hold with some $c>0$ and $0<s<1$, then
\begin{equation}\label{est.hypo2}
   \mathcal E(f^0_{D,\lambda})-\mathcal
       E(f_\lambda)
       \leq
        8\|h_\rho\|^2_{L_{\rho_X}^2}\left(\frac{1}{m^2\lambda}+\frac{\tilde{c}}{m\lambda^s}\right)
         \log^2\frac{4}\delta+\lambda^2\|h_\rho\|^2_{L_{\rho_X}^2}
\end{equation}
holds  with confidence $1-\delta/2$, where $\tilde{c}$ is a constant
depending only on $c$.
\end{lemma}

\begin{proof}  Based on (\ref{regularitycondition}), (\ref{equality})
and (\ref{Appr}), we get
\begin{eqnarray}
  &&\mathcal E(f^0_{D,\lambda})-\mathcal
       E(f_\lambda)
        =
        \|f^0_{D,\lambda}-f_\rho\|^2_{L_{\rho_X}^2}-
       \|f_{\lambda}-f_\rho\|_{L_{\rho_X}^2}^2 \nonumber\\
       &\leq&
       2\|f^0_{D,\lambda}-f_\lambda\|^2_{L_{\rho_X}^2}+ \|f_{\lambda}-f_\rho\|^2_{L_{\rho_X}^2}
       \leq
       2\|f^0_{D,\lambda}-f_\lambda\|^2_{L_{\rho_X}^2}+\lambda^2\|h_\rho\|^2_{L_{\rho_X}^2}.
       \label{hypo 1.1.1}
\end{eqnarray}
Due to (\ref{Def.three functions}), we have
$$
     f^0_{D,\lambda}-f_\lambda
     =
     (L_{K,D}-L_K)(L_K+\lambda I)^{-1}f_\rho
$$
Thus, $\kappa\leq 1$ implies
\begin{eqnarray*}
      &&\|f^0_{D,\lambda}-f_\lambda\|_{L_{\rho_X}^2}
       =
      \|L_K^{1/2}(L_{K,D}-L_K)(L_K+\lambda I)^{-1}L_Kh_\rho\|_K\\
      &\leq&
       \|(L_{K,D}-L_K)(L_K+\lambda I)^{-1/2}\|\|L_K^{1/2}h_\rho\|_K
       =
       \|(L_K+\lambda
      I)^{-1/2}(L_{K,D}-L_K)\|\|h_\rho\|_{L_{\rho_X}^2}.
\end{eqnarray*}
But (\ref{eigenvalue value decaying}) and the definition of
$\mathcal N(\lambda)$ yield
\begin{eqnarray*}
       \mathcal N(\lambda)
       &=&
       \sum_{\ell=1}^\infty\frac{\mu_\ell}{\lambda+\mu_\ell}
       \leq
       \sum_{\ell=1}^\infty\frac{c\ell^{-1/s}}{\lambda+c\ell^{-1/s}}
       =
       \sum_{\ell=1}^\infty\frac{c}{c+\lambda\ell^{1/s}}\\
       &\leq&
     \int_0^\infty\frac{c}{c+\lambda t^{1/s} }dt
     \leq
     c \lambda^{-s}\left(\int_0^c\frac1cdt+\int_c^\infty
     t^{-1/s}dt\right)
       =
       \tilde{c} \lambda^{-s},
\end{eqnarray*}
where $\tilde{c}:=c+ \frac{c}{1-s}c^{1-1/s}$. It then follows from
Lemma \ref{Lemma:operator inequality} that with confidence at least
$1-\delta/2$, there holds
$$
      \|f^0_{D,\lambda}-f_\lambda\|_{L_{\rho_X}^2}  \leq \frac{2}{\sqrt{m}}
     \left\{\frac{1}{\sqrt{m\lambda}} +\sqrt{\tilde{c}\lambda^{-s}}\right\} \|h_\rho\|_{L_{\rho_X}^2} \log
     \frac4\delta.
$$
Inserting the above inequality into (\ref{hypo 1.1.1}) and noting
$(a+b)^2\leq 2a^2+2b^2$, we get that with confidence at least
$1-\delta/2$, there holds
$$
   \mathcal E(f^0_{D,\lambda})-\mathcal
       E(f_\lambda)
       \leq
        8\|h_\rho\|^2_{L_{\rho_X}^2}\left(\frac{1}{m^2\lambda}+\frac{\tilde{c}}{m\lambda^s}\right)
         \log^2\frac{4}\delta+\lambda^2\|h_\rho\|^2_{L_{\rho_X}^2}.
$$
This completes the proof of Lemma \ref{Lemma:Hypo 1.1}.
\end{proof}

The bound of $\mathcal E_D(f_{D,k})
      -\mathcal E_D(f_{D,\lambda}^0)$
is more technical. At first, we use the well known Bernstein
inequality \citep{Shi2011} to present a tight bound of
$\|f^0_{D,\lambda}\|_{\ell_1}$.

\begin{lemma}\label{Lemma:BERNSTEIN}
 Let $\xi$ be a random variable on
$\mathcal Z$ with variance $\gamma^2$ satisfying $|\xi- E[\xi]|\leq
M_\xi$ for some constant $M_\xi$. Then for any $0<\delta<1$, with
confidence $1-\delta$, we have
$$
             \frac1m\sum_{i=1}^m\xi(z_i)-
             E[\xi]\leq\frac{2M_\xi\log\frac1\delta}{3m}+\sqrt{\frac{2\gamma^2\log\frac1\delta}{m}}.
$$
\end{lemma}

With the help of Lemma \ref{Lemma:BERNSTEIN}, we derive the following $\ell_1$ norm estimate for $f_{D,\lambda}^0$.
\begin{lemma}\label{Lemma:l1norm for media}
Let $0<\delta<1$. Under $\kappa\leq 1$ and
Assumption \ref{Assumption:smoothness},  with confidence $1-\delta/2$,
there holds
\begin{equation}\label{l1norm for media}
      \|f^0_{D,\lambda}\|_{\ell_1}
     \leq \frac{1}m\sum_{i=1}^m|f^0_\lambda(x_i)|
     \leq
     \left(1+\frac{4 }{3m\sqrt{\lambda}}\log\frac2\delta
     +\sqrt{\frac{2}{m}\log\frac2\delta}\right)\|h_\rho\|_{L_{\rho_X}^2}.
\end{equation}
\end{lemma}

\begin{proof} We at first bound
$\|f^0_{D,\lambda}\|_{\ell_1}-\|f_\lambda^0\|_{L_{\rho_X}^1}$ by
using Lemma \ref{Lemma:BERNSTEIN}. Let
 $\xi_1=|f^0_\lambda(x)|$. We then have
$$
  E(\xi_1)=\|f_\lambda^0\|_{L_{\rho_X}^1},\qquad    \|f_{D,\lambda}^0\|_{\ell_1}
  =  \frac1m\sum_{i=1}^m|f_0^\lambda(x_i)|=\frac1m\sum_{i=1}^m\xi_1(x_i).
$$
Due to $\kappa\leq 1$, we have
$\|f\|_\infty\leq \|f\|_K$ for arbitrary $f\in\mathcal H_K$. Then it
follows from
  $\|h_\rho\|_{L_{\rho_X}^2}= \|L_K^{1/2}h_\rho\|_K$, $f_\lambda^0\in\mathcal H_K$,
(\ref{Def.three functions}) and Assumption \ref{Assumption:smoothness} that
\begin{eqnarray}\label{f0lambda uniform}
      \|f^0_\lambda\|_\infty \leq\|f^0_\lambda\|_K
       \leq
      \frac1{\sqrt{\lambda}}\|(L_K+\lambda
      I)^{-1/2}L_Kh_\rho\|_K
      \leq
      \frac1{\sqrt{\lambda}}\|h_\rho\|_{L_{\rho_X}^2}.
\end{eqnarray}
Furthermore,
\begin{equation}\label{bound.f0}
      \|f^0_\lambda\|_{L_{\rho_X}^2}
     =\frac1\lambda\|\lambda(L_K+\lambda I)^{-1}f_\rho\|_{L_{\rho_X}^2}
     =\frac1\lambda\|f_\lambda-f_\rho\|_{L_{\rho_X}^2}
     \leq \|h_\rho\|_{L_{\rho_X}^2}.
\end{equation}
Then,
$$
       E[\xi_1^2]=\|f_\lambda^0\|_{L_{\rho_X}^2}^2\leq\|h_\rho\|_{L_{\rho_X}^2}^2.
$$
 Hence, for arbitrary $0<\delta<1$, Lemma
\ref{Lemma:BERNSTEIN} with $\xi=\xi_1=|f_0^\lambda(x)|$, $M_{\xi}=
\frac2{\sqrt{\lambda}}\|h_\rho\|_{L_{\rho_X}^2}$ and $\gamma^2\leq
\|h_\rho\|_{L_{\rho_X}^2}^2$ yields that with confidence
  $1-\delta/2$, there holds
\begin{equation}\label{est.hypo1}
     \frac1m\sum_{i=1}^m|f_0^\lambda(x_i)|
     -\|f_\lambda^0\|_{L_{\rho_X}^1}
     \leq
     \left(\frac{4}{3m\sqrt{\lambda}}\log\frac2\delta
     +\sqrt{\frac{2}{m}\log\frac2\delta}\right)\|h_\rho\|_{L_{\rho_X}^2}.
\end{equation}
  But (\ref{bound.f0}) implies
\begin{equation}\label{f0lambda l2}
      \|f_\lambda^0\|_{L_{\rho_X}^1}\leq\|f^0_\lambda\|_{L_{\rho_X}^2}
      \leq\|h_\rho\|_{L_{\rho_X}^2}.
\end{equation}
Plugging (\ref{f0lambda l2}) into (\ref{est.hypo1}),   with
confidence $1-\delta/2$, there holds
$$
      \|f^0_{D,\lambda}\|_{\ell_1}
     \leq \frac{1}m\sum_{i=1}^m|f^0_\lambda(x_i)|
     \leq
     \left(1+\frac{4 }{3m\sqrt{\lambda}}\log\frac2\delta
     +\sqrt{\frac{2}{m}\log\frac2\delta}\right)\|h_\rho\|_{L_{\rho_X}^2}.
$$
This completes the proof of Lemma \ref{Lemma:l1norm for media}.
\end{proof}

Then, we use Lemma \ref{Lemma:Hypo 1.1}, Lemma \ref{Lemma:l1norm for media} and Theorem \ref{Theorem:numerical convergence rate} to bound $\mathcal
H(D,\lambda,k)$.

\begin{proposition}\label{Proposition:hypothesis error}
Let $0<\delta<1$ and $f_{D,k}$ be defined by (\ref{Step 2:Line search}) with $l_k=c_0\log (k+1)$ and $\alpha_k=\frac2{k+2}$. Under $\kappa\leq 1$, Assumptions
\ref{Assumption:smoothness} and \ref{Assumption:eigenvalue decay}
with some $c>0$ and $0<s<1$, if   $\lambda=m^{-1/(s+1)}$ and  $m^{-1/(s+1)}\log^2\frac4\delta\leq 1$,
then with confidence at least $1-\delta$, there holds
\begin{equation}\label{hypothesis error estimate}
    \mathcal H(D,\lambda,k)\leq C'\left(m^{-\frac1{1+s}}\log^2\frac4\delta+
   k^{-1}\right),
\end{equation}
where $C'$ is a constant depending only on $c$, $c_0$ and
$\|h_\rho\|_{L_{\rho_X}^2}$.
\end{proposition}

{\bf Proof.} It follows from Lemma \ref{Lemma:Hypo 1.1} with $\lambda=m^{-1/(s+1)}$ that there
exists a subset $\mathcal Z^m_{1,\delta}\subset\mathcal Z^m$ with
measure $1-\delta/2$ such that for arbitrary $D\in \mathcal
Z^m_{1,\delta}$
\begin{equation}\label{h final 1}
   \mathcal E(f^0_{D,\lambda})-\mathcal
       E(f_\lambda)
       \leq
        (9+8\tilde{c})m^{-\frac1{1+s}}\log^2\frac4\delta\|h_\rho\|^2_{L_{\rho_X}^2}.
\end{equation}
Since $m^{-1/(s+1)}\log^2\frac4\delta\leq 1$ and
$\lambda=m^{-\frac1{1+s}}$, Lemma \ref{Lemma:l1norm for media} shows
that there exists a subset $\mathcal Z^m_{2,\delta}\subset\mathcal
Z^m$ with measure $1-\delta/2$ such that for arbitrary $D\in
\mathcal Z^m_{2,\delta}$, there holds
\begin{equation}\label{l1bound final}
       \frac{1}m\sum_{i=1}^m|f^0_\lambda(x_i)|
     \leq
       4\|h_\rho\|_{L_{\rho_X}^2}.
\end{equation}
Let $k^*$ be the smallest integer satisfying  $l_{k^*}\geq
\frac1m\sum_{i=1}^m|f^0_\lambda(x_i)|$. $l_k=c_0\log(k+1)$ then
implies  that for arbitrary $D\in \mathcal Z^m_{2,\delta}$, there
holds
$$
        k^*\leq
        \exp\left\{\frac{4\|h_\rho\|_{L_{\rho_X}^2} }{c_0}\right\}+1,
$$
and
$$
         l_{k^*}\leq \log
         \left(\exp\left\{\frac{4\|h_\rho\|_{L_{\rho_X}^2}}{c_0}\right\}+2\right)
         \leq \log3+ \frac{4\|h_\rho\|_{L_{\rho_X}^2}}{c_0}.
$$
 Hence, we have from Theorem \ref{Theorem:numerical convergence rate}
with $h=\frac1m\sum_{i=1}^mf_\lambda^0(x_i)K_{x_i}$ that for
arbitrary $D\in \mathcal Z^m_{2,\delta}$, there holds
\begin{equation}\label{h final 2}
               \mathcal E_D(f_{D,k})
      -\mathcal E_D(f_{D,\lambda}^0) \leq \tilde{C}'_1
             k^{-1},
\end{equation}
where
\begin{eqnarray*}
      \tilde{C}'_1&:=&
     32
     \left(\exp\left\{\frac{4\|h_\rho\|_{L_{\rho_X}^2}}{c_0}\right\}+1\right)\\
     &&
     \left[16M^2\left(\exp\left\{\frac{4\|h_\rho\|_{L_{\rho_X}^2}}{c_0}\right\}+3\right)+4 \left(
      \log3+ \frac{4\|h_\rho\|_{L_{\rho_X}^2}}{c_0}\right)^2\right].
\end{eqnarray*}
Plugging (\ref{h final 1}) and (\ref{h final 2}) into
(\ref{Def.Hypothesis error}), for arbitrary $D\in  \mathcal
Z^m_{1,\delta}\cap \mathcal Z^m_{2,\delta}$, we obtain
\begin{equation}\label{h final 1}
   \mathcal H(D,\lambda,k)\leq C'\left(m^{-\frac1{1+s}}\log^2\frac4\delta+
   k^{-1}\right),
\end{equation}
where
$$
       C'=\max\left\{C'_1,(9+8\tilde{c})\|h_\rho\|^2_{L_{\rho_X}^2}\right\}.
$$
This proves Proposition \ref{Proposition:hypothesis error} by noting
the measure of $\mathcal Z^m_{1,\delta}\cap \mathcal Z^m_{2,\delta}$
is $1-\delta$. $\Box$

\subsubsection{Sample error estimate}

To bound the sample error, we need the following oracle inequality, which is a modified version of \citep[Theorem 7.20]{Steinwart2008}. We present its proof in Section 6.3.

\begin{theorem}\label{Theorem:concentration inequality}
Let $0<\delta<1$ and $R>0$. If $|y_i|\leq M$, $\kappa\leq 1$ and Assumption
\ref{Assumption:eigenvalue decay} holds  with some $c>0$ and $0<s<1$, then with confidence
$1-\delta$, there holds
\begin{eqnarray}\label{Concentration inequality}
        &&\left|\{\mathcal E(f)-\mathcal
       E(f_\rho)\}-\{\mathcal E_D(f)-\mathcal E_D(f_\rho)\}\right|
      \leq
     \frac12(\mathcal E(f)-\mathcal E(f_\rho))+\frac{32(3M+
     R)^2\log\frac1\delta}{3m} \nonumber \\
     &+& \bar{C}(3M+R)^2\max\left\{ (\mathcal E(f)-\mathcal
     E(f_\rho))^\frac{1-s}2
     m^{-\frac12},
     m^{-\frac1{1+s}}\right\},\qquad \forall f\in B_{K,R},
\end{eqnarray}
where $B_{K,R}:=\{f\in\mathcal H_K: \|f\|_K\leq R\}$ and $\bar{C}$ is a constant depending only on $s$ and $c$.
\end{theorem}

We then use   the oracle inequality
established in Theorem \ref{Theorem:concentration inequality} to
derive the upper bound of $\mathcal S_2(D,\lambda)$.

\begin{proposition}\label{Proposition:sample error 1}
Let $0<\delta<1$. Under $|y_i|\leq M$, $\kappa\leq1$, Assumptions
\ref{Assumption:smoothness} and \ref{Assumption:eigenvalue decay}
with some $c>0$ and $0<s<1$  If  $m^{-1/(s+1)}\log^2\frac6\delta\leq 1$ and
$\lambda=m^{-1/(s+1)}$, then with confidence at least $1-\delta$,
there holds
\begin{equation}\label{hypothesis error estimate}
    \mathcal S_2(D,\lambda)\leq  C_1m^{-\frac{1}{1+s}}\log^{2}\frac6\delta,
\end{equation}
where $C_1$ is a constant independent of $m$ or $\delta$.
\end{proposition}

\begin{proof} Due to (\ref{equality}), we have
$$
     \mathcal E(f^0_{D,\lambda})-\mathcal
       E(f_\rho)
       =
      \mathcal E(f^0_{D,\lambda})-\mathcal
       E(f_\lambda)+\|f_\lambda-f_\rho\|_{L_{\rho_X}^2}^2.
$$
Then, it follows from (\ref{Appr}) and
Lemma \ref{Lemma:Hypo 1.1} that with confidence $1-\delta$, there
holds
\begin{equation}\label{add5}
   \mathcal E(f^0_{D,\lambda})-\mathcal
       E(f_\rho)
       \leq
        8\|h_\rho\|^2_{L_{\rho_X}^2}\left(\frac{1}{m^2\lambda}+\frac{\tilde{c}}{m\lambda^s}\right)
         \log^2\frac{4}\delta+2\lambda^2\|h_\rho\|^2_{L_{\rho_X}^2}.
\end{equation}
 Since $m^{-1/(s+1)}\log^2\frac4\delta\leq 1$ and
$\lambda=m^{-\frac1{1+s}}$, it follows from (\ref{l1bound final})
that for   arbitrary $D\in \mathcal Z^m_{2,\delta}$, there holds
$$
       \frac{1}m\sum_{i=1}^m|f^0_\lambda(x_i)|
     \leq
       4\|h_\rho\|_{L_{\rho_X}^2}.
$$
Furthermore, (\ref{add5}) together with
$\lambda=m^{-\frac1{1+s}}$   shows that there exists a subset $\mathcal Z^m_{3,\delta}$
of $\mathcal Z^m$ with measure $1-\frac\delta2$ such that for all
$D\in \mathcal Z^m_{3,\delta}$, there holds
$$
      \max\left\{(\mathcal E(f_{D,\lambda}^0)-\mathcal E(f_\rho))^\frac{1-s}2
     m^{-\frac12},
     m^{-\frac1{1+s}}\right\}\leq \tilde{c}_1
     m^{-\frac{1}{1+s}}\log^{1-s}\frac4\delta,
$$
where
$\tilde{c}_1:=[8(1+\tilde{c})+2]^{\frac{1-s}2}\|h_\rho\|_{L_{\rho_X}^2}^{1-s}$.
Setting  $R= 4\|h_\rho\|_{L_{\rho_X}^2}$ , it then follows from
Theorem \ref{Theorem:concentration inequality} that there is a
subset $\mathcal Z^m_{4,\delta}$ of $\mathcal Z^m$ with measure
$1-\frac\delta2$ such that for each $D\in \mathcal
Z^m_{2,\delta}\cap\mathcal Z^m_{3,\delta}\cap\mathcal
Z^m_{4,\delta}$, there holds
\begin{eqnarray*}
        &&\mathcal S_2(D,\lambda)
      \leq
      [4(1+\tilde{c})+2]\|h_\rho\|^2_{L_{\rho_X}^2}
         m^{-\frac{1}{s+1}}\log^2\frac{4}\delta +\frac{32(3M+
     R)^2\log\frac2\delta}{3m} \nonumber \\
     &+& \bar{C}(3M+R)^2\tilde{c}_1
     m^{-\frac{1}{1+s}}\log^{1-s}\frac4\delta
     \leq
     C_1m^{-\frac{1}{1+s}}\log^{2}\frac4\delta,
\end{eqnarray*}
where
$$
     C_1:=[4(1+\tilde{c})+1]\|h_\rho\|^2_{L_{\rho_X}^2}
     +\frac{32(3M+
     4\|h_\rho\|_{L_{\rho_X}^2})^2}{3}+ \bar{C}(3M+4\|h_\rho\|_{L_{\rho_X}^2})^2\tilde{c}_1.
$$
This proves Proposition \ref{Proposition:sample error 1} by scaling
$\frac{3\delta}{2}$ to $\delta$.
\end{proof}

In the following, we aim to derive the estimate for $\mathcal
S_1(D,k)$.

\begin{proposition}\label{Proposition:sample error 2}
Let $0<\delta<1$ and $f_{D,k}$ be defined by (\ref{Step 2:Line
search}) with $\alpha_k=\frac2{k+2}$ and $l_k=c_0\log(k+1)$. Under $|y_i|\leq M$, $\kappa\leq 1$, Assumptions \ref{Assumption:smoothness} and
\ref{Assumption:eigenvalue decay} with some $c>0$ and $0<s<1$, if
$m^{-1/(s+1)}\log^2\frac6\delta\leq 1$ and
\begin{equation}\label{case 1 ass}
       \mathcal E(f_{D,k})-\mathcal E(f_\rho)\geq
       m^{-\frac1{1+s}},
\end{equation}
then
\begin{eqnarray*}
         \mathcal S_1(D,k)
      &\leq&
     \frac12(\mathcal E(f_{D,k})-\mathcal E(f_\rho))+\frac{32(3M+
     c_0\log (k+1))^2\log\frac2\delta}{3m}\\
     &+&\bar{C}(3M+c_0\log (k+1))^2(\mathcal E(f_{D,k})-\mathcal
     E(f_\rho))^\frac{1-s}2
     m^{-\frac12}
\end{eqnarray*}
holds with confidence   $1-\delta$, where $C_1$ is a constant
independent of $m$ or $\delta$.
\end{proposition}

\begin{proof} Due to (\ref{case 1 ass}),  there holds
\begin{equation}\label{sam 2.1}
      \max\left\{ (\mathcal E(f_{D,k})-\mathcal
     E(f_\rho))^\frac{1-s}2
     m^{-\frac12},
     m^{-\frac1{1+s}}\right\}
     =
     (\mathcal E(f_{D,k})-\mathcal
     E(f_\rho))^\frac{1-s}2
     m^{-\frac12}.
\end{equation}
But Theorem \ref{Theorem:numerical convergence rate} together with $\kappa\leq 1$  implies
$f_{D,k}\in B_R\subset B_{K,R}$ with $R=c_0\log (k+1)$. Then it follows from
Theorem \ref{Theorem:concentration inequality} that there exists a
subset $\mathcal Z^m_{5,\delta}$ of $\mathcal Z^m$ with measure
$1-\frac\delta2$ such that for each $D\in \mathcal Z^m_{5,\delta}$,
there holds
\begin{eqnarray*}
         \mathcal S_1(D,k)
      &\leq&
     \frac12(\mathcal E(f_{D,k})-\mathcal E(f_\rho))+\frac{32(3M+
     c_0\log (k+1))^2\log\frac2\delta}{3m}\\
     &+&\bar{C}(3M+c_0\log (k+1))^2(\mathcal E(f_{D,k})-\mathcal
     E(f_\rho))^\frac{1-s}2
     m^{-\frac12}.
\end{eqnarray*}
This completes the proof of Proposition \ref{Proposition:sample
error 2}.
\end{proof}

\subsubsection{Generalization error analysis}
In this part, we use Propositions \ref{Proposition:hypothesis error}, \ref{Proposition:sample error 1} and \ref{Proposition:sample error 2} to prove Theorem \ref{Theorem:Bound in Probability}.

\noindent
{\bf Proof of Theorem \ref{Theorem:Bound in Probability}.}
If (\ref{case 1 ass}) does not hold, then we obtain (\ref{oracle}) directly. In the rest,
we
are only concerned with $k$ for which  (\ref{case 1 ass}) holds. It follows from Propositions \ref{Proposition:hypothesis error}, \ref{Proposition:sample error 1} and \ref{Proposition:sample error 2} to prove Theorem \ref{Theorem:Bound in Probability}, (\ref{Appr}) with
$\lambda=m^{-\frac1{s+1}}$ and (\ref{Error decomposition for l1})  that with confidence $1-\delta$, there
holds
\begin{eqnarray*}
    &&\mathcal E(f_{D,k})-\mathcal
   E(f_\rho)
   \leq m^{-\frac2{s+1}}\|h_\rho\|_{L_{\rho_X}^2}^2
   +
   C'\left(m^{-\frac1{1+s}}\log^2\frac{18}\delta+
   k^{-1}\right)\\
   &+&
   C_1m^{-\frac{1}{1+s}}\log^{2}\frac{18}\delta
   +
   \frac12(\mathcal E(f_{D,k})-\mathcal E(f_\rho))+\frac{32(3M+
     c_0\log (k+1))^2\log\frac{18}\delta}{3m}\\
     &+&\bar{C}(3M+c_0\log (k+1))^2(\mathcal E(f_{D,k})-\mathcal
     E(f_\rho))^\frac{1-s}2
     m^{-\frac12}\\
   &\leq&
   2(C_1+C'+\|h_\rho\|_{L_{\rho_X}^2}^2+384M^2)m^{-\frac1{s+1}}\log^{2}\frac{18}\delta
   +2C'k^{-1}\\
   &+&
   2\bar{C}(3M+c_0\log (k+1))^2(\mathcal E(f)-\mathcal
     E(f_\rho))^\frac{1-s}2
     m^{-\frac12}+44c_0^2\frac{\log^2(k+1)}{m}\log^{2}\frac{18}\delta.
\end{eqnarray*}
Since
$m^{-\frac1{1+s}}=\left(m^{-\frac1{1+s}}\right)^{\frac{1-s}2}m^{-\frac12}$,
$k\geq m^\frac1{s+1}$ and (\ref{case 1 ass}) holds, we have
\begin{eqnarray*}
     \mathcal E(f_{D,k})-\mathcal
   E(f_\rho)
   \leq
    C_1' \log (k+1) (\mathcal E(f_{D,k})-\mathcal
   E(f_\rho))^{\frac{1-s}2} m^{-1/2} \log^{2}\frac{18}\delta,
\end{eqnarray*}
where
$$
     C_1':=4\max\{(C_1+2C'+\|h_\rho\|_{L_{\rho_X}^2}^2+384M^2)+6M\bar{C},c_0\bar{C}
      +22c_0^2\}.
$$
Hence, with confidence $1-\delta$, there holds
\begin{eqnarray*}
     \mathcal E(f_{D,k})-\mathcal
   E(f_\rho)
   \leq
    (C'_1)^{\frac2{1+s}}\left(\log(k+1)\right)^{\frac2{1+s}}  m^{-\frac1{1+s}}
     \log^{4}\frac{18}\delta.
\end{eqnarray*}
This proves Theorem \ref{Theorem:Bound in Probability} with
$C:=(C'_1)^{\frac2{1+s}}$.
 \hfill\BlackBox

\subsection{Proof of Theorem \ref{Theorem:concentration inequality}}

Our oracle inequality is built upon the  eigenvalue decaying
assumption (\ref{eigenvalue value decaying}). We  at first connect it with
the well known entropy number  defined in Definition
\ref{Definition:entropy number} below.

\begin{definition}\label{Definition:entropy number}
Let $E$ be a Banach space and $A\subset E$ be a bounded subset. Then
for $i\geq1,$ the $i$-th entropy number $e_i(A,E)$ of $A$ is the
infimum over all $\varepsilon>0$ for which there exist
$t_1,\dots,t_{2^{i-1}}\in A$ with
$A\subset\bigcup_{j=1}^{2^{i-1}}(t_j+\varepsilon B_E)$, where $B_E$
denotes the closed unit ball of $E$. Moreover, the $i$-th entropy
number of a bounded linear operator $\mathcal T:E\rightarrow F$ is
$e_i(\mathcal T):=e_i(\mathcal TB_E,F),$ where $\mathcal
TB_E:=\{\mathcal Tf:f\in B_E\}$.
\end{definition}

We also need the   following
two lemmas, which can be found in \cite[Theorem 15]{SteinwartHS}
and \cite[Corollary 7.31]{Steinwart2008}, respectively.

\begin{lemma}\label{Lemma:entropy and eigen}
Let $\{\mu_i\}_{i=1}^\infty$ be the set eigenvalues of the operator
$L_K:L_{\rho_X}^2\rightarrow L_{\rho_X}^2$ arranging in a decreasing
order. For arbitrary $0<p<1$, there exists a constant $c'_p$
depending only on $p$ such that
$$
    \sup_{i\leq j}i^{\frac1p}e_i(id:\mathcal H_K\rightarrow
    L_{\rho_X}^2)\leq c'_p\sup_{i\leq
    j}i^{\frac1p}\mu^{1/2}_i,\qquad\forall j\geq 1.
$$
\end{lemma}

\begin{lemma}\label{Lemma:entropy for empirical and rho}
Assume that there exist constants $0<p<1$ and $a\geq 1$ such that
$$
          e_i(id:\mathcal H_K\rightarrow L_{\rho_X}^2)\leq
          ai^{-\frac1{2p}},\qquad i\geq 1.
$$
Then there exists a constant $c_p>0$ depending only on $p$ such that
$$
        E[e_i(id:\mathcal H_K\rightarrow \ell^2(D_X))]\leq
        c_pi^{-\frac1{2p}},
$$
where $\ell^2(D_X)$ denotes the empirical $\ell^2$ space with
respect to $(x_1,\dots,x_m)$.
\end{lemma}

With the help of Lemmas \ref{Lemma:entropy and eigen} and
\ref{Lemma:entropy for empirical and rho}, we  derive
 the following upper bound for the empirical entropy number in
 expectation.

\begin{lemma}\label{Lemma:entropy estimate}
If $|y_i|\leq M$, $\kappa\leq 1$,   and Assumption
\ref{Assumption:eigenvalue decay} holds with some $c>0$ and $0<s<1$, then
\begin{equation}\label{entropy.3}
      E[e_i(\mathcal F_R,\ell^2(D))]
        \leq c_sc_s'\sqrt{c}R(2M+2R)i^{-\frac1{2s}},
\end{equation}
where
\begin{equation}\label{Def.FR}
 \mathcal F_R:=\{\phi_f(x)=(f(x)-y)^2-(f(x)-f_\rho(x))^2:
          f\in B_{K,R}\},
\end{equation}
and $\ell^2(D)$ denotes the empirical $\ell^2$ space  with respect
to $(z_1,\dots,z_m)$.
\end{lemma}

\begin{proof} Due to (\ref{eigenvalue value decaying}) and  Lemma
\ref{Lemma:entropy and eigen} with $p=s$, we have for arbitrary
$j\geq 1$,
$$
    j^\frac{1}{s}e_j(id:\mathcal H_K\rightarrow L_{\rho_X}^2)
    \leq
    \sup_{i\leq j}i^{\frac1s}e_i(id:\mathcal H_K\rightarrow
    L_{\rho_X}^2)\leq c_s'\sqrt{c}j^\frac{1}{2s},
$$
which implies
$$
       e_i(id:\mathcal H_K\rightarrow
    L_{\rho_X}^2)\leq c_s'\sqrt{c}i^{-\frac1{2s}},\qquad\forall
    i=1,2,\dots.
$$
This  together with Lemma \ref{Lemma:entropy for empirical and rho}
yields
\begin{equation}\label{entropy.1}
        E[e_i(id:\mathcal H_K\rightarrow \ell^2(D_X))]\leq
        c_sc_s'\sqrt{c}i^{-\frac1{2s}},\qquad\forall i=1,2,\dots.
\end{equation}
For
arbitrary $f\in  B_{K,R}$, there exists an $f^*\in B_{K,1}$
such that $f=Rf^*$. Let $f_1,\dots,f_{2^{n-1}}$ be an $\varepsilon$
net of $B_{K,1}$. Then there exists an $f_{j^*}$ such that
$$
      \frac1m\sum_{i=1}^m(f(x_i)-Rf_{j^*}(x_i))^2
      =
      \frac1m\sum_{i=1}^m(Rf^*(x_i)-Rf_{j^*}(x_i))^2
      \leq
      R^2\varepsilon^2.
$$
Thus, $Rf_1,\dots,Rf_{2^{n-1}}$ is an $R\varepsilon$ net of $B_{K,R}$.
This  together with  (\ref{entropy.1})  implies
\begin{equation}\label{entropy.2}
       E[e_i(B_{K,R},\ell^2(D_X))]\leq   RE[e_i(B_{K,1},\ell^2(D_X))]
        \leq c_sc_s'\sqrt{c}Ri^{-\frac1{2s}}.
\end{equation}
For arbitrary $\phi_f\in\mathcal F_R$, there exists an $f\in B_{K,R}$
such that $\phi_f(x)=(f(x)-y)^2-(f_\rho(x)-y)^2$. Then  there exists
an $f_{j^*}$ with $1\leq j^*\leq 2^{n-1}$ such that
\begin{eqnarray*}
      &&\frac1m\sum_{i=1}^m(\phi_f(x_i)-((f_{j^*}(x_i)-y_i)^2-(f_\rho(x_i)-y_i)^2))^2\\
     & =&
      \frac1m\sum_{i=1}^m(f(x_i)-f_{j^*}(x_i))^2(f(x_i)+f_{j^*}(x_i)-2y_i)^2
       \leq
      (2M+2R)^2R^2\varepsilon^2,
\end{eqnarray*}
where we used   $|y_i|\leq M$ in above
estimates. Hence, it follows from (\ref{Def.FR}) and
(\ref{entropy.2}) that
$$
      E[e_i(\mathcal F_R,\ell^2(D))]\leq
       (2M+2R)E[e_i(B_R,\ell^2(D_X))]
        \leq c_sc_s'\sqrt{c}R(2M+2R)i^{-\frac1{2s}}.
$$
This completes the proof of Lemma \ref{Lemma:entropy estimate}.
\end{proof}
We then present  a close relation   between the  empirical entropy
number  and  empirical Rademacher average \cite[Definitions
7.8\&7.9]{Steinwart2008}.

\begin{definition}\label{Definition:Rademacher}
Let $(\Theta,\mathcal C,\nu)$ be a probability space and
$\epsilon_i:\Theta\rightarrow\{-1,1\}$, $i=1,\dots,m$, be
independent random variables with
$\nu(\epsilon_i=1)=\nu(\epsilon_i=-1)=1/2$ for all $i=1,\dots,m$.
Then, $\epsilon_1,\dots,\epsilon_m$ is called a Rademacher sequence
with respect to $\nu$. Assume $\mathcal H\subset\mathcal M(\mathcal
Z)$ be a non-empty set with $\mathcal M(\mathcal
Z)$  the set  of measurable functions on $\mathcal Z$. For $D =(z_1,\dots,z_m)\in \mathcal Z^m$,
the $m$-th empirical Rademacher average of $\mathcal H$ is defined
by
$$
          Rad_D(\mathcal H,m):=E\left[\sup_{h\in\mathcal
          H}\left|\frac1m\sum_{i=1}^m\epsilon_ih(z_i)\right|\right].
$$
\end{definition}

 The following   lemma
which were proved in \cite[Lemma 7.6,Theorem 7.16]{Steinwart2008} show that the
upper bound of the empirical entropy number of $\mathcal F_R$ implies an upper bound
of the empirical Rademacher average.

\begin{lemma}\label{Lemma:entropy to radam}
Suppose that
 there exist constants $B_1\geq0$ and $\sigma_1\geq0$ such that
$\|h\|_\infty\leq B_1$ and $E[h^2]\leq\sigma_1^2$ for all
$h\in\mathcal H$. Furthermore, assume that for a fixed $m\geq1$
there exist constants $p\in(0,1)$ and $a\geq B_1$ such that
$$
      E[e_i(\mathcal F_R,\ell^2(D))]\leq a i^{-\frac1{2p}},\qquad
      i\geq 1.
$$
Then there exist constants $C_p$ and $C_p'$ depending only on $p$
such that
$$
     E [Rad_D(\mathcal
     F_R,m)]\leq\max\left\{C_pa^p\sigma_1^{1-p}m^{-\frac12},C_p'
     a^{\frac{2p}{1+p}}B_1^{\frac{1-p}{1+p}}m^{-\frac1{1+p}}\right\}.
$$
\end{lemma}

Furthermore, the following lemma proved in
\cite[Lemma 7.6,Proposition 7.10]{Steinwart2008}, presents the role of the empirical Rademacher average in empirical process.

\begin{lemma}\label{Lemma:Symmetrization}
For arbitrary $m\geq 1$
we have
$$
         E\left[\sup_{h\in\mathcal
         F_R}\left|E[h]-\frac1m\sum_{i=1}^mh(z_i)\right|\right]
         \leq 2 E[Rad_D(\mathcal F_R,m)].
$$
\end{lemma}

Based on the above two lemmas, we can derive the following bound, which plays an important role in our analysis.

\begin{lemma}\label{Lemma:median}
If $\kappa\leq 1$, $|y_i|\leq M$ and Assumption
\ref{Assumption:eigenvalue decay} holds with some $c>0$ and $0<s<1$,  then there exists a
constant $\widetilde{C}$ depending only on $s$ and $c$ such that
\begin{equation}\label{median bound}
     E\left[\sup_{\phi_f\in\mathcal
         F_{R}}\left|E[\phi_f]-\frac1m\sum_{i=1}^m\phi_f(z_i)\right|\right]
          \leq
      \widetilde{C} (3M+R)^2\max\left\{ (E[\phi_f])^{\frac{1-s}2}
     m^{-\frac12},
     m^{-\frac1{1+s}}\right\}.
\end{equation}
 \end{lemma}
\begin{proof} For arbitrary  $\phi_f\in\mathcal F_R$,   we   have from
(\ref{Def.FR}) $|y_i\leq M$ and $\kappa\leq 1$  that
$$
       \|\phi_f\|_\infty\leq (3M+ R)^2=:B_1,\qquad E[\phi_f^2]\leq
       (3M+ R)^2E[\phi_f]=:\sigma_1^2.
$$
  Let $\bar{c}\geq 1$ be the smallest constant such that
  $\overline{c}c_sc_s'\sqrt{c}\geq1$, that is,
$$
         B_1=(3M+ R)^2\leq\bar{c}c_sc_s'\sqrt{c}(3M+ R)^2=:a.
$$
Then, (\ref{entropy.3}) implies
$$
      E[e_i(\mathcal F_R,\ell^2(D))]
        \leq a i^{-\frac1{2s}}.
$$
Thus, it follows from Lemma
\ref{Lemma:entropy to radam} with
$p=s$ that
\begin{eqnarray*}
     E [Rad_D(\mathcal
     F_R,m)] \leq
       C'  (3M+R)^2\max\left\{  (E[\phi_f])^{\frac{1-s}2}
     m^{-\frac12},
     m^{-\frac1{1+s}}\right\},
\end{eqnarray*}
where $C'=
\max\{C_s(\bar{c}c_sc'_s\sqrt{c})^s,C_s'(\bar{c}c_sc_s'\sqrt{c})^{\frac{2s}{1+s}}\}.$
Based on Lemma \ref{Lemma:Symmetrization}, we then get
\begin{eqnarray*}
     E\left[\sup_{\phi_f\in\mathcal
         F_R}\left|E[\phi_f]-\frac1m\sum_{i=1}^m\phi_f(z_i)\right|\right]
          \leq
       2C'  (3M+R)^2\max\left\{  (E[\phi_f])^{\frac{1-s}2}
     m^{-\frac12},
     m^{-\frac1{1+s}}\right\}.
\end{eqnarray*}
  This  proves Lemma \ref{Lemma:median} with
$\widetilde{C}=2C'$.
\end{proof}

For $\varepsilon>0$, define
\begin{eqnarray}\label{Def.GRvarepsilon}
        \mathcal
        G_{R,\varepsilon}:=\left\{g_{\phi_f,\varepsilon}
        =\frac{E[\phi_f]-\phi_f}{E[\phi_f]+\varepsilon}:\phi_f\in
        \mathcal F_R\right\}.
\end{eqnarray}
  Lemma \ref{Lemma:median} implies the following estimate.

\begin{lemma}\label{Lemma:peeling}
If $|y_i|\leq M$, $\kappa\leq 1$ and Assumption
\ref{Assumption:eigenvalue decay} holds with some $c>0$ and $0<s<1$, then for  arbitrary
$\varepsilon\geq\inf_{\phi_f\in \mathcal F_R}E[\phi_f]$, there
exists a constant $\widetilde{C}_1$ depending only on $c$ and $s$
such that
\begin{eqnarray*}
       E\left[\sup_{ g_{\phi_f,\varepsilon} \in \mathcal
        G_{R,\varepsilon}}\left|\frac1m\sum_{i=1}^mg_{\phi_f,\varepsilon}(z_i)\right|\right]
        \leq
      \widetilde{C}_1\frac{(3M+R)^2}{\varepsilon}\max\left\{ \varepsilon^{\frac{1-s}2}
     m^{-\frac12},
     m^{-\frac1{1+s}}\right\}.
\end{eqnarray*}

\end{lemma}

\begin{proof} For arbitrary $\phi_f\in\mathcal F_R$, it follows from
(\ref{Def.FR}) that  $E[\phi_f]=\mathcal E(f)-\mathcal E(f_\rho)\geq
0$. Then,
\begin{eqnarray*}
      &&\sup_{\phi_f\in \mathcal
      F_R}\left|\frac{E[\phi_f]-\frac1m\sum_{i=1}^m\phi_f(z_i)}{E[\phi_f]
      +\varepsilon}\right|
       \leq
      \sup_{\phi_f\in \mathcal
      F_R,E[\phi_f]\leq \varepsilon}
      \frac{\left|E[\phi_f]-\frac1m\sum_{i=1}^m\phi_f(z_i)\right|}{\varepsilon}\\
      &+&
      \sum_{j=0}^\infty\sup_{\phi_f\in \mathcal
      F_R,4^j\varepsilon\leq E[\phi_f]\leq 4^{j+1}\varepsilon}
      \frac{\left|E[\phi_f]-\frac1m\sum_{i=1}^m\phi_f(z_i)\right|}{4^j\varepsilon+\varepsilon},
\end{eqnarray*}
where we used the convention $\sup\varnothing:=0$. Let
$r\geq\inf_{\phi_f\in \mathcal F_R}E[\phi_f]$ be arbitrary real
number. It follows from
  (\ref{median bound})  that
\begin{eqnarray*}
     E\left[\sup_{\phi_f\in\mathcal
         F_R,E[\phi_f]\leq r}\left|E[\phi_f]-\frac1m\sum_{i=1}^m\phi_f(z_i)\right|\right]
          \leq
       \widetilde{C}  (3M+R)^2\max\left\{  r^{\frac{1-s}2}
     m^{-\frac12},
     m^{-\frac1{1+s}}\right\}.
\end{eqnarray*}
Repeating the above inequality with $r=4^j\varepsilon$ and $\varepsilon\geq\inf_{\phi_f\in
\mathcal F_R}E[\phi_f]$ for
$j=0,1,\dots,$   we get from the  above two estimates that
\begin{eqnarray*}
       &&E\left[\sup_{\phi_f\in \mathcal
      F_R}\left|\frac{E[\phi_f]-\frac1m\sum_{i=1}^m\phi(z_i)}{E[\phi_f]+\varepsilon}\right|\right]
       \leq
      \frac{ \widetilde{C}(3M+R)^2\max\left\{ \varepsilon^{\frac{1-s}2}
     m^{-\frac12},
     m^{-\frac1{1+s}}\right\}}{\varepsilon}\\
       &+&
      \sum_{j=0}^\infty
      \frac{\widetilde{C}(3M+R)^2\max\left\{ (4^{j+1}\varepsilon)^{\frac{1-s}2}
     m^{-\frac12},
     m^{-\frac1{1+s}}\right\}}{4^j\varepsilon+\varepsilon}\\
     &\leq&
     \frac{\widetilde{C}(3M+R)^2}{\varepsilon}\max\left\{ \varepsilon^{\frac{1-s}2}
     m^{-\frac12},
     m^{-\frac1{1+s}}\right\}
     \left(1+\sum_{j=0}^\infty\frac{4^{\frac{(1-s)(j+1)}2}}{4^j+1}\right).
\end{eqnarray*}
Since
$$
    \sum_{j=0}^\infty\frac{4^{\frac{(1-s)(j+1)}2}}{4^j+1}\leq2^{1-s}
    \sum_{j=0}^\infty2^{(-s-1)j}
     =\frac{2^{1-s}}{1-2^{-s-1}}\leq 4,
$$
we get from (\ref{Def.GRvarepsilon}) that
\begin{eqnarray*}
      &&E\left[\sup_{ g_{\phi_f,\varepsilon} \in \mathcal
        G_{R,\varepsilon}}\left|\frac1m\sum_{i=1}^mg_{\phi_f,\varepsilon}(z_i)\right|\right]
        =E\left[\sup_{\phi_f\in \mathcal
      F_R}\left|\frac{E[\phi_f]-\frac1m\sum_{i=1}^m\phi_f(z_i)}{E[\phi_f]+\varepsilon}\right|\right]\\
      &\leq&
      4\widetilde{C}\frac{(3M+R)^2}{\varepsilon}\max\left\{
      \varepsilon^\frac{1-s}2
     m^{-\frac12},
     m^{-\frac1{1+s}}\right\}.
\end{eqnarray*}
This completes the proof of Lemma \ref{Lemma:peeling} with $
     \widetilde{C}_1:=4\widetilde{C}.$
\end{proof}

Lemma \ref{Lemma:peeling} builds the estimate in expectation. To derive similar bound in probability, we need the following concentration inequality, which is a simplified version of  Talagrand's inequality and can be found in
\cite[Theorem 7.5, Lemma 7.6]{Steinwart2008}

\begin{lemma}\label{Lemma:Talagrand inequality}
Let  $B\geq0$
and $\sigma\geq0$ be constants such that
$E[g^2]\leq\sigma^2$  and $\|g\|_\infty\leq B$ for all $g\in\mathcal
G_{R,\varepsilon}$.
Then, for all $\tau>0$ and all $\gamma>0$, we have
\begin{eqnarray}\label{Talagrand inequality}
       && P\left(\left\{z\in \mathcal Z^m: \sup_{ g_{\phi_f,\varepsilon} \in \mathcal
        G_{R,\varepsilon}}\left|\frac1m\sum_{i=1}^mg_{\phi_f,\varepsilon}(z_i)\right|
        \geq (1+\gamma)E\left[\sup_{ g_{\phi_f,\varepsilon} \in \mathcal
        G_{R,\varepsilon}}\left|\frac1m\sum_{i=1}^mg_{\phi_f,\varepsilon}(z_i)\right|
        \right]\right.\right.\nonumber\\
        &+&
        \left.\left.\sqrt{\frac{2\tau\sigma^2}{m}}+\left(\frac23+\frac1\gamma\right)
        \frac{\tau B}{m}\right\}\right)\leq e^{-\tau}.
\end{eqnarray}
\end{lemma}

Now, we are in a position to prove Theorem
\ref{Theorem:concentration inequality} by using Lemma
 \ref{Lemma:Talagrand inequality} and lemma \ref{Lemma:peeling}.

\noindent
{\bf Proof of Theorem \ref{Theorem:concentration inequality}.} For
arbitrary $f\in B_{K,R}$, we have $E[\phi_f]=\mathcal E(f)-\mathcal
E(f_\rho)\geq 0$ with $\phi_f=(y-f(x))^2-(y-f_\rho(x))^2 \in\mathcal
F_R$. Furthermore, $|y_i|\leq M$ and $ \|f\|_\infty\leq \|f\|_{K}\leq R$ yield
$\|E[\phi_f]-\phi_f\|_\infty\leq 2(R+3M)^2$. For arbitrary
$\varepsilon\geq\inf_{\phi_f\in \mathcal F_R}E[\phi_f]$ and
$g_{\phi_f,\varepsilon}\in\mathcal G_{\phi_f,\varepsilon}$, there
exists a $\phi_f\in\mathcal F_R$ such that $
g_{\phi_f,\varepsilon}=\frac{E[\phi_f]-\phi_f}{E[\phi_f]+\varepsilon}$.
Then, we get
\begin{equation}\label{1.con.1}
     \|g_{\phi_f,\varepsilon}\|_\infty\leq\frac{2(3M+R)^2}{\varepsilon}=:B,
\end{equation}
and
\begin{equation}\label{1.con.2}
     E[g^2_{\phi_f,\varepsilon}]\leq
     \frac{E[\phi_f^2]}{(E[\phi_f+\varepsilon)^2}
     \leq  \frac{(3M+
     R)^2E[\phi_f]}{(E[\phi_f+\varepsilon)^2} \leq \frac{(3M+
     R)^2}{\varepsilon}.
\end{equation}
Then Lemma
\ref{Lemma:Talagrand inequality} with $\gamma=1$ and $\varepsilon\geq\inf_{\phi_f\in
\mathcal F_R}E[\phi_f]$,   Lemma \ref{Lemma:peeling} with $\varepsilon\geq\inf_{\phi_f\in
\mathcal F_R}E[\phi_f]$,   (\ref{1.con.1})
and (\ref{1.con.2}) that with confidence at least $1-e^{-\tau}$,
there holds
\begin{eqnarray*}
        &&\sup_{\phi_f\in \mathcal
        F_R}\left|\frac{E[\phi_f]-\frac1m\sum_{i=1}^m\phi_f(z_i)}{E[\phi_f]+\varepsilon}\right|
         =
        \sup_{g_{\phi_f,\varepsilon}\in \mathcal
        G_{R,\varepsilon}}\left|\frac1m\sum_{i=1}^mg_{\phi_f,\varepsilon}(z_i)\right|\\
        &\leq&
         2E\left[\sup_{ g_{\phi_f,\varepsilon} \in \mathcal
        G_{R,\varepsilon}}\left|\frac1m\sum_{i=1}^mg_{\phi_f,\varepsilon}(z_i)\right|\right]
        +\sqrt{\frac{2\tau(3M+
     R)^2}{m\varepsilon}}+\frac{10(3M+
     R)^2\tau}{3m\varepsilon}\\
     &\leq&
     2\widetilde{C}_1\frac{(3M+R)^2}{\varepsilon}\max\left\{ \varepsilon^\frac{1-s}2
     m^{-\frac12},
     m^{-\frac1{1+s}}\right\} +\sqrt{\frac{2\tau(3M+
     R)^2}{m\varepsilon}}+\frac{10(3M+
     R)^2\tau}{3m\varepsilon}.
\end{eqnarray*}
For arbitrary $f\in B_{K,R}$, set $\tau=\log\frac1\delta$ and
$\varepsilon=\mathcal E(f)-\mathcal E(f_\rho)\geq\inf_{\phi_f\in \mathcal
F_R}E[\phi_f]$. It follows from $\mathcal E(f)-\mathcal
E_D(f)=E[\phi_f]-\frac1m\sum_{i=1}^m\phi_f(z_i)$ that , with
confidence $1-\delta$, there holds
\begin{eqnarray*}
        &&\left|\mathcal E(f)-\mathcal
       E(f_\rho)+\mathcal E_D(f_\rho)-\mathcal E_D(f_\rho)\right|
       \leq \sqrt{\frac{8(3M+
     R)^2(\mathcal E(f)-\mathcal E(f_\rho))\log\frac1\delta}{m}}\\
       &+&
      \frac{20(3M+
     R)^2\log\frac1\delta}{3m}+4\widetilde{C}_1(3M+R)^2\max\left\{ (\mathcal E(f)-\mathcal E(f_\rho))^\frac{1-s}2
     m^{-\frac12},
     m^{-\frac1{1+s}}\right\}\\
     &\leq&
     \frac12(\mathcal E(f)-\mathcal E(f_\rho))+\frac{32(3M+
     R)^2\log\frac1\delta}{3m}\\
     &+& 4\widetilde{C}_1(3M+R)^2\max\left\{ (\mathcal E(f)-\mathcal
     E(f_\rho))^\frac{1-s}2
     m^{-\frac12},
     m^{-\frac1{1+s}}\right\},
\end{eqnarray*}
where we used the element inequality $\sqrt{ab}\leq\frac12(a+b)$ for
$a,b>0$ in the last inequality. This completes the proof of Theorem
\ref{Theorem:concentration inequality} with
$\bar{C}=4\widetilde{C}$. \hfill  \BlackBox

{\section*{Acknowledgments}
The work of Yao Wang is supported partially by the National Key Research and Development Program of China (No. 2018YFB1402600), and the National Natural Science Foundation of China (Nos. 11971374, 61773367). The work of Xin Guo is supported partially by  Research Grants Council of Hong Kong [Project No. PolyU 15305018]. The work of
Shao-Bo Lin is supported partially by the National Natural Science Foundation of China (Nos. 61876133, 11771012).}

\end{document}